 \documentclass[a4paper]{article}

\usepackage{hyperref}       
\usepackage{url}            
\usepackage{amssymb}
\usepackage{amsmath}
\usepackage{amsthm}
\usepackage{graphicx}
\usepackage{natbib}
\usepackage{marvosym}

\bibpunct{(}{)}{;}{a}{,}{,}

\newtheorem{theorem}{Theorem}
\newtheorem{lemma}{Lemma}
\newtheorem{corollary}{Corollary}
\newtheorem{definition}{Definition}

\newcommand{\RDS}{R_{\DS}}
\newcommand{\RDSj}{R_{\DSj}}
\newcommand{\RDT}{R_{\DT}}
\newcommand{\RD}{R_{D}}

\usepackage{amsmath}
\usepackage{amssymb}
\usepackage{enumerate}
\usepackage{color}
\usepackage[table]{xcolor} 
\usepackage{url}
\usepackage{booktabs}
\usepackage{mathtools}
\usepackage{xspace}

\usepackage{upgreek}

\usepackage{hyperref}

\definecolor{mydarkblue}{rgb}{0,0.08,0.45} 
\hypersetup{
  colorlinks   = true, 
  urlcolor     = mydarkblue, 
  linkcolor    = mydarkblue, 
  filecolor    = mydarkblue,
  citecolor    = mydarkblue 
}

\makeatletter
\def\captionof#1#2{{\def\@captype{#1}#2}}
\makeatother

\newcommand{\LP}{\left(}
\newcommand{\LB}{\left[}

\newcommand{\RP}{\right)}
\newcommand{\RB}{\right]}

\newcommand{\degree}{\ensuremath{^\circ}}
\newcommand{\hdh}{\Hcal\!\Delta\!\Hcal}

\newcommand{\Erf}{\ensuremath{\textbf{\small Erf}}}
\newcommand{\vbf}{\ensuremath{\mathbf{v}}}

\newcommand{\xbf}{\ensuremath{\mathbf{x}}}

\newcommand{\PSj}{P_{S_j}}
\newcommand{\DSj}{D_{S_j}}

\newcommand{\multiP}{\PS^{v}}
\newcommand{\multiD}{\DS^{v}}
\newcommand{\multiS}{S^{v}}

\newcommand{\pii}{{\boldsymbol{\uppi}}}
\newcommand{\PS}{P_S}
\newcommand{\PT}{P_T}

\newcommand{\DS}{D_S}
\newcommand{\DT}{D_T}
\newcommand{\Sj}{{S_j}}

\newcommand{\RPv}{R_{\multiP}}
\newcommand{\RSv}{R_{\multiS}}

\newcommand{\RSj}{R_{S_j}}

\newcommand{\RPS}{R_{\PS}}
\newcommand{\RPSj}{R_{\PSj}}

\newcommand{\RPT}{R_{\PT}}

\newcommand{\ST}{{\langle S, T \rangle}}

\newcommand{\RS}{R_{S}}
\newcommand{\RT}{R_{T}}

\newcommand{\Pcal}{\ensuremath{\mathcal{P}}}
\newcommand{\Hcal}{\ensuremath{\mathcal{H}}}

\newcommand{\Dcal}{\ensuremath{\mathcal{D}}}

\newcommand{\Fcal}{\ensuremath{\mathcal{F}}}
\newcommand{\Scal}{\ensuremath{\mathcal{S}}}

\newcommand{\GQ}{G_\posterior}

\newcommand{\posterior}{\rho}
\newcommand{\prior}{\pi}

\newcommand{\sign}{\operatorname{sign}}

\newcommand{\desPI}{\operatorname{dis}_\prior}

\newcommand{\des}{\operatorname{dis}_\posterior}
\newcommand{\desw}{\operatorname{dis}_{\posterior_\wb}}
\newcommand{\disc}{\operatorname{disc}}

\newcommand{\R}{\mathbb{R}}
\newcommand{\I}{\mathbf{I}}

\definecolor{vertee}{rgb}{0,.5,0}

\newcommand{\eqdef}{\overset{{\mbox{\rm\tiny def}}}{=}}

\newcommand{\expo}[1]{
    e^{#1}
}

\newcommand{\argmin}[1]{
    \underset{#1}{\mathrm{argmin}}\
}
\newcommand{\prob}[1]{
    \underset{#1}{\mathrm{Pr}}\
}

\newcommand{\esp}[1]{
    \underset{#1}{\mathrm{\bf E}}\
}

\newcommand{\argmindevant}[1]{
    {\mathrm{argmin}}_{#1}\
}

\newcommand{\espdevant}[1]{
    \mathrm{{\bf E}}_{#1}\,
}

\newcommand{\xb}{\mathbf{x}}

\newcommand{\KL}{{\rm KL}}
\newcommand{\kl}{{\rm kl}}
\newcommand{\zo}{0\textrm{-}\!1}

\newcommand{\zoloss}{\mathcal{L}_{_{\zo}}}

\newcommand{\loss}{\mathcal{L}}

\newcommand{\xbs}{{\xb^s}}
\newcommand{\xbt}{{\xb^t}}

\newcommand{\wbf}{\mathbf{w}}
\newcommand{\wb}{\mathbf{w}}
\newcommand{\ab}{{\pmb{\alpha}}}
\newcommand{\sgn}{\mathrm{sgn}}

\newcommand{\Phidis}{\Phi_{\rm dis}}
\newcommand{\Phic}{\Phi_{\rm cvx}}

\newcommand{\ePSj}{{e_{\PSj}}}

\newcommand{\ePS}{{e_{\PS}}}
 \newcommand{\ePT}{{e_{\PT}}} 
 \newcommand{\eP}{{e_{P}}}

\newcommand{\PBDA}{{\small PBDA}\xspace}
\newcommand{\SVM}{{\small SVM}\xspace}
\newcommand{\msda}{{\small mSDA}\xspace}

\usepackage{fancyhdr}
\usepackage{fancybox}
\newcommand{\ver}{2}
\title{PAC-Bayesian Theorems for Domain Adaptation with Specialization to Linear Classifiers} 
\author{Pascal Germain$^1$ \and Amaury Habrard$^2$ \and Fran\c cois Laviolette$^3$ \and Emilie Morvant$^2$ \and
$\mbox{}^1$ \small INRIA, SIERRA Project-Team, 75589 Paris, France, et D.I., \\ \small \'Ecole Normale Sup\'erieure, 75230 Paris, France
\and
$\mbox{}^2$ \small Univ Lyon, UJM-Saint-Etienne, CNRS, Institut d'Optique Graduate School, \\ \small  Laboratoire Hubert Curien UMR 5516, F-42023, Saint-Etienne, France
\and
$\mbox{}^3$ \small D\'epartement d'informatique et de g\'enie logiciel\\
\small
Universit\'e Laval, Qu\'ebec, Canada
}

\begin{document}

\maketitle
\begin{center}
\it
This report is a long version of our paper entitled \emph{A PAC-Bayesian Approach for Domain Adaptation with Specialization to Linear Classifiers} published in the proceedings of the International Conference on Machine Learning (ICML) 2013. We improved our main results, extended our experiments, and proposed an extension to multisource domain adaptation.
\end{center}


\begin{abstract}In this paper, we provide two main contributions in PAC-Bayesian theory for domain adaptation where the objective is to learn, from a source distribution, a well-performing majority vote  on a different target distribution. 
On the one hand, we propose an improvement of the previous approach proposed by \cite{pbda}, that relies on a novel distribution pseudodistance based on a disagreement averaging, allowing us to  
  derive a new tighter PAC-Bayesian domain adaptation bound for the stochastic Gibbs classifier.
We specialize it to linear classifiers, and design a learning algorithm which shows interesting results on a synthetic problem and on a popular sentiment annotation task. 
On the other hand, we generalize these results to multisource domain adaptation allowing us to take into account different source domains.
This study opens the door to tackle domain adaptation tasks by making use of all the PAC-Bayesian tools.
\end{abstract}



\section{Introduction}

As human beings, we learn from what we saw before. 
Think about our education process: when a student attends to a new course, he has to make use of the knowledge he acquired during previous courses. However, in machine learning the most common assumption is based on the fact that the learning and test data are drawn from the same probability distribution.
This strong assumption may be clearly irrelevant for a lot of real tasks including those where we desire to adapt a model from one task to another one.
For instance, a spam filtering system suitable for one user can be poorly adapted to another who receives significantly different emails.
In other words, the learning data associated with one or several users could be unrepresentative of the test data coming from another one.
This enhances the need to design methods for adapting a classifier from learning (source) data to test (target) data.
One solution to tackle this issue is to consider the {\it domain adaptation} framework\footnote{See the surveys proposed by \citet{JiangSurvey08,Quionero-Candela:2009}; and~\citet{Margolis2011}.}, which arises when the distribution generating the target data (the {\it target domain}) differs from the one generating the source data (the {\it source domain}).
In such a situation, it is well known that domain adaptation is a hard and challenging task even under strong assumptions~\citep{BenDavid12,David-AISTAT10,Ben-DavidU14}. 
Note that domain adaptation with learning data coming from different source domains is referred to as multisource or multiple sources domain adaptation \citep{crammer2007learning,mansour2009domain,BenDavid-MLJ2010}.

Among the existing approaches in the literature to address domain adaptation, the instance weighting-based methods allow one to deal with the covariate-shift problem \citep[\emph{e.g.,}][]{HuangSGBS-nips06,sugiyama2008direct}, where source and target domains diverge only in their marginals, {\it i.e.}, they share the same labeling function.
Another technique is to exploit self-labeling procedures, where the objective is to transfer the source labels to the target unlabeled points ({\it e.g.}, \citet{BruzzoneM10S,habrard2013iterative,PVMinCq}.  
A third solution is to learn a new common representation from the unlabeled part of source and target data. Then, a standard supervised learning algorithm can be executed on the source labeled instances ({\it e.g.}, \citet{glorot2011domain,Chen12}).

The work presented in this paper stands into a popular class of approaches, which relies on a distance between the source distribution and the target distribution. Such distance depends on the set $\Hcal$ of hypotheses (or classifiers)  considered by the learning algorithm. The intuition behind this approach is that one must look for a set  $\Hcal$ that minimizes the distance 
while preserving good performances on the source data; if the distributions are close under this measure, then generalization ability may be ``easier'' to quantify.
In fact, defining such a measure to quantify how much the domains are related is a major issue in domain adaptation.
For example, in the context of  binary classification with the $\zo$ loss function, \citet{BenDavid-MLJ2010}; and~\citet{BenDavid-NIPS06} 
have considered the {\small $\hdh$}-divergence between the marginal distributions. This quantity is based on the maximal disagreement between two classifiers, allowing them to deduce  a domain adaptation generalization bound based on the VC-dimension theory. The discrepancy distance proposed by \citet{Mansour-COLT09} generalizes this divergence to real-valued functions and more general losses, and is used to obtain a generalization bound based on the Rademacher complexity. In this context, \citet{CortesM11,CortesM14} have specialized the minimization of the discrepancy to regression with kernels.
In these situations, domain adaptation can be viewed as a multiple trade-off between the complexity of the hypothesis class $\Hcal$, the adaptation ability of~$\Hcal$ according to the divergence between the marginals, and the empirical source risk.
Moreover, other measures have been exploited under different assumptions, such as the R{\'e}nyi divergence suitable for importance weighting \citep{MansourMR09}, or the measure proposed by \citet{Zhang12} which takes into account the source and target true labeling, or the Bayesian ``divergence prior''  \citep{li2007bayesian} which favors classifiers closer to the best source model.
However, a majority of methods prefer to perform a two-step approach: {\it (i)}~first construct a suitable representation by minimizing the divergence, then {\it (ii)}~learn a model on the source domain in the new representation space.

The novelty of our contribution is to explore the PAC-Bayesian framework to tackle domain adaptation 
in a binary classification situation without target labels (sometimes called {\it unsupervised domain adaptation}).
Given a prior distribution over a family of classifiers~$\Hcal$, PAC-Bayesian theory \citep[introduced by][]{Mcallester99a} focuses on algorithms that output a posterior distribution $\posterior$ over $\Hcal$ ({\it i.e.}, a $\posterior$-average over $\Hcal$) rather than just a single classifier $h\in\Hcal$.
Following this principle, we propose a pseudometric which evaluates the domain divergence according to  the $\posterior$-average disagreement of the classifiers over the domains. 
This disagreement measure shows many advantages. First, it is ideal for the PAC-Bayesian setting, since it is expressed as a $\posterior$-average over $\Hcal$. 
Second, we prove that it is always lower than the popular {\small $\hdh$}-divergence. 
Last but not least, our measure can be easily estimated from samples.
Indeed, based on this disagreement measure, we derived in a previous work \citep{pbda} a first PAC-Bayesian domain adaptation bound expressed as a $\posterior$-averaging.
In this paper, we provide a new version of this result, that does not change the philosophy supported by the previous bound, but clearly 
improves the theoretical result: The domain adaptation bound is now tighter and easier to interpret.
Thanks to this new result, we also derive\footnote{In this paper, we were very keen to improve the readability of our proofs, 
particularly those provided by~\citet{pbda} as supplementary material. The proof techniques may be of independent interest.}
 three new PAC-Bayesian domain adaptation generalization bounds.
Then, in contrast to the majority of methods that perform a two-step procedure, we design an algorithm tailored to linear classifiers, called \PBDA, which jointly minimizes the multiple trade-offs implied by the bounds.
The first two quantities being, as usual in the PAC-Bayesian approach, the complexity of the majority vote measured by a Kullback-Leibler divergence and  the empirical risk measured by the $\posterior$-average errors on the source sample. 
The third quantity corresponds to our domain divergence and assesses the capacity of the posterior distribution to distinguish some structural difference between the source and target samples.
Finally, we extend our results to domain adaptation with multiple sources by considering a mixture of different source domains as done by \citet{BenDavid-MLJ2010}.

The rest of the paper is structured as follows.
Section~\ref{sec:notations} deals with two seminal works on domain adaptation. 
The PAC-Bayesian framework is then recalled in Section~\ref{sec:pacbayes}.
Note that for the sake of completeness, we provide for the first time the explicit derivation of the algorithm {\small PBGD3} \citep{germain2009pac} tailored to linear classifiers in supervised learning.
Our main contribution, which consists in a domain adaptation bound suitable for PAC-Bayesian learning, is presented in Section~\ref{sec:our_bound}.
Then, we derive our new algorithm for PAC-Bayesian domain adaptation in Section~\ref{sec:dapbgd}, that we experiment in Section~\ref{sec:expe}.
Afterwards, we generalize this analysis to multisource domain adaptation in Section~\ref{sec:multisource}.
Before concluding in Section~\ref{sec:conclu}, we discuss two important points in Section~\ref{sec:discussions}: {\it (i)} two different results for the multisource setting that imply open-questions for deriving new algorithms, and {\it (ii)} the comparison between our new result and the one provided in \citet{pbda}.


\section{Domain Adaptation Related Works}
\label{sec:notations}
\label{sec:da}

In this section, we review the two seminal works in domain adaptation that are based on a divergence measure between the domains \citep{BenDavid-MLJ2010,BenDavid-NIPS06,Mansour-COLT09}.

\subsection{Notations and Setting}
We consider domain adaptation for binary classification tasks where $X  \subseteq   \mathbb{R}^d$ is the input space of dimension $d$ and \mbox{$Y  =  \{-1, +1\}$} is the label set.
The {\it source domain} $\PS$ and the {\it target domain} $\PT$ are two different distributions over  $X\times Y$ (unknown and fixed), 
$\DS$ and $\DT$ being the respective marginal distributions over $X$. 
We tackle the challenging task where we have no target labels. A learning algorithm is then provided with a {\it labeled source sample} $S = \{(\xbf^s_i,y^s_i)\}_{i=1}^{m}$ consisting of $m$ examples drawn {\it i.i.d.}\footnote{{\it i.i.d.} stands for {\it independent and identically distributed}.} from $\PS$, and an {\it unlabeled target sample} $T = \{\xbf^t_j\}_{j=1}^{m'}$ consisting of $m'$ examples drawn {\it i.i.d.} from $\DT$.
Note that, we denote the distribution of a $m$-sample by $(\PS)^m$.
We suppose that $\Hcal$ is a set of hypothesis functions for $X$ to $Y$.
The {\it expected source error}  and the {\it expected target error}  of $h\in\Hcal$ over $\PS$, respectively $\PT$, are the probability that $h$ errs on the entire distribution $\PS$, respectively~$\PT$,
$$
\RPS(h) \ \eqdef \, \esp{(\xbf^s,y^s) \sim \PS} \zoloss\big( h(\xb^s), y^s \big)\,,\quad\mbox{and}\quad \RPT(h) \ \eqdef \, \esp{(\xbf^t,y^t) \sim \PT} \zoloss\big( h(\xb^t), y^t \big)\,,
$$
where $\zoloss(a,b)  \eqdef  \I[a \ne b]$ is the $\zo$ loss function which returns $1$ if $a \ne  b$ and $0$ otherwise.  The {\it empirical source error} $\RS(\cdot)$ on the learning sample $S$ is
$$
\RS(h)  \ \eqdef \ \frac{1}{m}\sum_{(\xbf^s,y^s) \in S} \zoloss\big( h(\xb^s), y^s \big)\,.
$$

The main objective in domain adaptation is then to learn---without target labels---a classifier $h\in\Hcal$ leading to the lowest expected target error $\RPT(h)$. \\

We also introduce the {\it expected source disagreement} $\RDS(h,h')$ and the {\it expected target disagreement}  $\RDT(h,h')$ of $(h',h)\in\Hcal^2$, which measure the probability that two classifiers $h$ and $h'$ do not agree on the respective marginal distributions, and are defined by
$$
\RDS(h,h') \ \eqdef\, \esp{\xbf^s\sim \DS} \zoloss\big( h(\xb^s), h'(\xb^s) \big)\,\quad\mbox{and}\quad\RDT(h,h') \ \eqdef\, \esp{\xbf^t\sim \DT} \zoloss\big( h(\xb^t), h'(\xb^t) \big)\,.
$$
The {\it empirical source disagreement} $\RS(h, h')$ on $S$ and the {\it empirical target disagreements} $\RT(h, h')$  on  $T$ are
$$
\RS(h,h') \ \eqdef\, \frac{1}{m}\sum_{\xbf^s\in S} \zoloss\big( h(\xb^s), h'(\xb^s) \big)\,\quad\mbox{and}\quad\RT(h,h') \ \eqdef\,  \frac{1}{m'}\sum_{\xbf^t\in T} \zoloss\big( h(\xb^t), h'(\xb^t) \big)\,.
$$
Note that, depending on the context, $S$ denotes either the source labeled sample $\{ (\xbf^s_i,y^s_i) \}_{i=1}^{m} $ or its unlabeled part $\{\xbf^s_i\}_{i=1}^{m} $.

Note also that the expected error $R_P(h)$ on a distribution $P$ can be viewed as a shortcut notation for the expected disagreement between a hypothesis $h$ and a labeling function $f_{P}$ that assigns the true label to an example description according with respect to  $P$. We have
\begin{equation*}
 R_P(h) \ =\ R_D(h, f_P) \ =\ 
 \esp{\xb\sim D} \zoloss\big( h(\xb), f_P(\xb) \big)\,,
\end{equation*}
where $D$ is the marginal distribution of $P$ over $X$.

\subsection{Necessity of a Domain Divergence}
\label{sec:necessity_dist}
The domain adaptation objective is to find a low-error target hypothesis, even if the target labels are not available. Even under strong assumptions, this task can be impossible to solve  \citep{BenDavid12,David-AISTAT10}. 
However, for deriving generalization ability in a domain adaptation situation 
(with the help of a domain adaptation bound), 
it is critical to make use of a divergence between the source and the target domains: the more similar the domains, the easier the adaptation appears.
Some previous works 
have proposed different quantities to estimate how a domain is close to another one \citep{Zhang12,BenDavid-MLJ2010,Mansour-COLT09,MansourMR09,BenDavid-NIPS06,li2007bayesian}.
Concretely, two domains $\PS$ and $\PT$ differ if their marginals  $\DS$ and $\DT$ are different, or if the source labeling function differs from the target one, or if both happen.
This suggests taking into account two divergences: one between $\DS$ and $\DT$ and one between the labeling.
If we have some target labels, we can combine the two distances as \citet{Zhang12}.
Otherwise, we preferably consider two separate measures, since it is impossible to estimate the best target hypothesis in such a situation. 
Usually, we suppose that the source labeling function is somehow related to the target one, then we look for a representation where the marginals $\DS$ and $\DT$ appear closer without losing performances on the source domain.

\subsection{Domain Adaptation Bounds for Binary Classification}

We now review the first two seminal works which propose domain adaptation bounds based on a marginal divergence.

First, under the assumption that there exists a hypothesis in $\Hcal$ that performs well on both the source and the target domain, \citet{BenDavid-MLJ2010}; and~\citet{BenDavid-NIPS06} 
have provided the following domain adaptation bound.
\begin{theorem}[\citet{BenDavid-MLJ2010,BenDavid-NIPS06}]
\label{theo:BenDavid}
Let ${\cal H}$ be a (symmetric\footnote{In a symmetric hypothesis space $\Hcal$, for every $h\in\Hcal$, its inverse $-h$ is also in $\Hcal$.}) hypothesis class. We have
\begin{equation}
\label{eq:da}
\forall h\in {\cal H},\  \RPT(h)\, \leq\, \RPS(h)+\tfrac{1}{2}d_{\hdh}(\DS,\DT) + \mu_{h^*}\, , 
\end{equation}
where  $$\tfrac{1}{2}d_{\hdh}(\DS,\DT) \ \eqdef \, \displaystyle        \sup_{\substack{(h,h')\in\mathcal{H}^2}}  \left|\RDT(h, h') - \RDS(h, h')\right|$$ is the {\small $\hdh$}-distance between the marginals $\DS$ and $\DT$, and
$$\mu_{h^*}\  \eqdef\  \RPS(h^{ *})+\RPT(h^{ *})$$ is the error of the best hypothesis overall, denoted $h^{ *}$, and defined by
$$h^{ *}  \ \eqdef\   \argmin{h\in{\cal H}}  \big( \RPS (h) + \RPT (h) \big)\,.$$ 
\end{theorem}
This bound depends on four terms. $\RPS(h)$ is the classical source domain expected error. 
 $\tfrac12 d_{\Hcal \Delta \Hcal}(\DS,\DT)$ depends on $\Hcal$ and corresponds to the maximum disagreement between two hypotheses of $\Hcal$.
In other words, it quantifies how hypothesis from $\Hcal$ can ``detect'' differences between these marginals: the lower this measure is for a given $\Hcal$, the better are the generalization guarantees.  The last term $\mu_{h^*}=\RPS(h^{ *})+\RPT(h^{ *})$  is related to the best hypothesis $h^{ *}$ over the domains and act as a quality measure of $\mathcal{H}$ in terms of labeling information. 
If $h^{ *}$ does not have a good performance on both the source and the target domain, then there is no way one can adapt from this source to this target. 
Hence, as pointed out by the authors, Equation~\eqref{eq:da}, together with the usual VC-bound theory, express a multiple trade-off between the accuracy of some particular hypothesis $h$, the complexity of  $\Hcal$, and the ``incapacity'' of hypotheses of $\Hcal$ to detect difference between the source and the target domain.

Second, \citet{Mansour-COLT09} have extended the {\small $\hdh$}-distance to the discrepancy divergence for regression and any symmetric loss  $\loss$ fulfilling the triangle inequality.
Given  \mbox{$\loss : [-1,+1]^2  \to  \R^+$}
such a loss, the discrepancy $\disc_{\loss}(\DS,\DT)$ between $\DS $ and~$\DT$~is  
$$\displaystyle \disc_{\loss}(\DS,\DT)  \eqdef    \sup_{\substack{(h,h')\in\Hcal^2}}  \Big| \esp{\xbf^t\sim\DT}     \loss(h(\xbf^t),h'(\xbf^t)) -     \esp{\xbf^s\sim\DS}      \loss(h(\xbf^s),h'(\xbf^s))\Big|\,.$$
Note that with the $\zo$ loss in binary classification, we have
$$\tfrac{1}{2}d_{\Hcal \Delta \Hcal}(\DS,\DT)   \,=\,  \disc_{\zoloss}(\DS,\DT)\,.$$
Even if these two divergences may coincide,  the following domain adaptation bound of \citet{Mansour-COLT09} differs from Theorem~\ref{theo:BenDavid}.

\begin{theorem}[\cite{Mansour-COLT09}]
\label{theo:Mansour}
Let ${\cal H}$ be a (symmetric) hypothesis class. We have
\begin{align}
\label{eq:dabounddisc} \forall h\in {\cal H},\ \RPT(h) - \RPT(h_T^*) \ \leq \   \RDS(h_S^*,h)  +\disc_{\zoloss}(\DS,\DT) + \nu_{(h_S^*,h_T^*)}&\,,
\end{align}
where
$$\nu_{(h_S^*,h_T^*)}\ \eqdef\ \RDS(h_S^*,h_T^*)$$
is the disagreement between the ideal hypothesis on the target and source domains defined respectively as 
\begin{align*}
h_T^* \ \eqdef\ \argmin{h\in\Hcal}  \RPT(h)\, , \quad \mbox{and}\quad\  h_S^*\ \eqdef \ \argmin{h\in\Hcal}\RPS(h)\,.
\end{align*}
\end{theorem}

In this context, Equation~\eqref{eq:dabounddisc} can be tighter\footnote{Equation \eqref{eq:da} can lead to an error term 3 times higher than Equation \eqref{eq:dabounddisc} in some cases \citep{Mansour-COLT09}.} since it bounds the difference between the target error of a classifier and the one of the optimal $h_T^*$. 
 This bound expresses a  trade-off between the disagreement (between $h$  and the best source hypothesis $h_S^*$), the complexity of  $\Hcal$ (with the Rademacher complexity), and---again---the ``incapacity'' of hypothesis to detect differences between the  domains. \\

To conclude, the domain adaptation bounds \eqref{eq:da} and \eqref{eq:dabounddisc} suggest that if the divergence between the domains is low, 
 a low-error classifier over the source domain might perform well on the target one.
These divergences compute the \emph{worst case} of the disagreement between a pair of hypothesis. We propose in Section~\ref{sec:our_bound} an \emph{average case} approach by making use of the essence of the PAC-Bayesian theory, which is known to offer tight generalization bounds \citep{Mcallester99a,germain2009pac,Parrado-Hernandez12}.


\section{PAC-Bayesian Theory in Supervised Learning}
\label{sec:pacbayes}

Let us now review the classical supervised binary classification framework called the PAC-Bayesian theory, first introduced by \citet{Mcallester99a}.
This theory succeeds to provide tight generalization guarantees on majority vote classifiers, without relying on any validation set.

Throughout this section, we adopt an algorithm design perspective: we interpret the various forms of the PAC-Bayesian theorem 
as a guide to derive new machine learning algorithms.
Indeed, the PAC-Bayesian analysis of domain adaptation provided in the forthcoming sections is oriented by the motivation of creating a new adaptive algorithms.

\subsection{Notations and Setting}

Traditionally, the PAC-Bayesian theory considers weighted majority votes over a set $\Hcal$ of binary hypothesis.
Given a prior distribution~$\prior$ over $\Hcal$ and a training set $S$, the learner aims at finding the posterior distribution $\posterior$ over $\Hcal$ leading to a $\posterior$-weighted majority vote  $B_\posterior$ (also called the Bayes classifier) with good generalization guarantees and defined by
$$B_\posterior(\xbf) \ \eqdef \ \sign\Big[\esp{h\sim \posterior} h(\xbf)\Big]\,.$$
Minimizing $\RPS(B_\posterior)$ the risk of $B_\posterior$ is known to be NP-hard.  In the PAC-Bayesian approach, it is replaced by the risk of the stochastic Gibbs classifier $G_\posterior$ associated with $\posterior$. In order to predict the label of an example~$\xbf$, the Gibbs classifier first draws a hypothesis $h$ from $\Hcal$ according to $\posterior$, then returns $h(\xbf)$ as label. Note that the error of the Gibbs classifier on a domain $\PS$ corresponds to the expectation of the errors over~$\posterior$:
\begin{align}
\label{eq:RGQ}
\RPS(G_\posterior) \ \eqdef \ \esp{h\sim\posterior} \RPS(h)\,.
\end{align}
In this setting, if $B_\posterior$ misclassifies $\xbf$, then at least half of the classifiers (under $ \posterior$) errs on $\xbf$. Hence, we  have $$ \RPS(B_\posterior) \ \leq  \  2\,\RPS(G_\posterior)\,.$$ 
Another result on the relation between $\RPS(B_\posterior)$ and $\RPS(G_\posterior)$ is the $C$-bound of \citet{Lacasse07} expressed as
\begin{align}
\label{eq:C-bound}
\RPS(B_\posterior) \ \leq\ 1-\frac{ \big(1-2\RPS(G_\posterior)\big)^2}{1-2\RDS(G_\posterior,G_\posterior)}\,,
\end{align}
where $\RDS(G_\posterior,G_\posterior)$ corresponds to the disagreement of the classifiers over $\posterior$:
\begin{align}
\label{eq:RGQGQ}
\RDS(G_\posterior,G_\posterior) \ \eqdef\, \esp{(h,h')\sim \posterior^2} \RDS(h,h')\,.
\end{align}
Equation~\eqref{eq:C-bound} suggests that for a fixed numerator, {\it i.e.}, a fixed risk of the Gibbs classifier, the best majority vote is the one with the lowest denominator, {\it i.e.},  with the greatest disagreement between its voters (see \citet{lmr-11} for further analysis).

Finally, we introduce the notion of {\it expected joint error} of a pair of classifiers $(h,h')$ drawn according to the distribution $\posterior$, defined as
\begin{equation}
\label{eq:eP}
\ePS(\GQ,\GQ) \  \eqdef\,   \esp{(h,h') \sim\posterior^2}\esp{(\xbf,y) \sim \PS} \zoloss\big( h(\xb), y \big) \times \zoloss\big( h'(\xb), y \big)\,.
\end{equation}

The PAC-Bayesian theory allows one to bound the expected error $\RPS(G_\posterior)$ in terms of two major quantities: the empirical error $\RS(G_\posterior)   =  \espdevant{h\sim\posterior} \RS(h)$ estimated on a sample~$S$ drawn {\it i.i.d.} from $\PS$  and the Kullback-Leibler divergence $\KL(\posterior\|\prior)  \eqdef  \espdevant{h\sim \posterior}  \ln \frac{\posterior(h)}{\prior(h)}$ (let us recall that $\prior$ and $\posterior$ are respectively the $\emph{prior}$ and the $\emph{posterior}$ distributions).
The three main PAC-Bayes theorems, that we present in the next section, have been proposed by \citet{Mcallester99a,Seeger02,Langford05}; and \citet{catoni2007pac}.

\subsection{Three Versions of the PAC-Bayesian Theorem}

First, let us consider the KL-divergence  $\kl(a\,\|\,b)$ between two Bernoulli distributions with success probability $a$ and $b$, defined by
\begin{equation*}
\kl(a\,\|\,b)\ \eqdef \ a\ln\frac{a}{b}+(1-a)\ln\frac{1-a}{1-b}\,.
\end{equation*}
\citet{Seeger02}; and \citet{Langford05} have derived the following PAC-Bayesian theorem in which the trade-off between the complexity and the risk is handled by  $\kl(\cdot\|\cdot)$.
\begin{theorem}[\citet{Seeger02,Langford05}] 
\label{thm:pacbayesseeger}
For any domain $\PS$ over  $X \times Y$,  any set of hypotheses $\Hcal$, and  any prior distribution $\prior$ over $\Hcal$, any $\delta\in(0,1]$,  with a probability at least $1-\delta$ over the choice of $S\sim (P_S)^m$, for every $\posterior$ over $\Hcal$, we have
\begin{equation*}
\kl\Big(\RS(G_{\posterior})\,\Big\|\, \RPS(G_{\posterior})\Big)\ \leq\   \frac{1}{m}\left[\KL(\posterior\,\|\,\prior)+ \ln  \frac{2\sqrt{m}}{\delta}\right].
\end{equation*}
\end{theorem}
This version of the PAC-Bayes theorem offers a tight bound, especially for low empirical risk.
However, due to the $\kl\left(\RS(G_{\posterior})\,\|\, \RPS(G_{\posterior})\right)$ term, this bound remains difficult to interpret: 
the link between the empirical risk $\RS(G_{\posterior})$ and the ``true'' risk $\RPS(G_{\posterior})$ is not given by a close form.
Thus, from an algorithmic point of view, finding the distribution $\posterior$ that minimizes the bound on $\RPS(G_{\posterior})$ given by Theorem~\ref{thm:pacbayesseeger}
might be a difficult task.

The following  version of the PAC-Bayes theorem, which was the first proposed \citep{Mcallester99a}, appears easier to interpret since
it links the terms $\RS(G_{\posterior})$ and $\RPS(G_{\posterior})$ by a linear relation.
Note that Theorem~\ref{thm:pacbayesallester} can be straightforwardly obtained from Theorem~\ref{thm:pacbayesseeger} using Pinsker's inequality: 
\begin{equation}
\label{eq:pinsker}
2(q-p)^2 \ \leq\ \kl(q\,\|\,p)\,.
\end{equation}

\begin{theorem}[\citet{Mcallester99a}] 
\label{thm:pacbayesallester}
For any domain $\PS$ over  $X \times Y$,  any set  of hypotheses $\Hcal$,  any prior distribution $\prior$ over $\Hcal$, and  any $\delta\in(0,1]$,  with a probability at least $1-\delta$ over the choice of $S\sim (P_S)^m$, for every $\posterior$ over $\Hcal$, we have
\begin{equation*}
\Big|\,\RPS(G_{\posterior})-\RS(G_{\posterior})\,\Big| \  \leq\  \sqrt{\frac{1}{2m}\left[\KL(\posterior\,\|\,\prior) + \ln  \frac{2\sqrt{m}}{\delta}\right]}\,.
\end{equation*}
\end{theorem}

Theorems~\ref{thm:pacbayesseeger} and~\ref{thm:pacbayesallester} suggest that, in order to minimize the expected risk, a learning algorithm should perform a trade-off between the empirical risk minimization $\RS(G_{\posterior})$ and KL-divergence minimization $\KL(\posterior\,\|\,\prior)$ (roughly speaking the complexity term).

The nature of this trade-off can be explicitly controlled in Theorem~\ref{thm:pacbayescatoni} below.
This PAC-Bayesian result, first proposed by \citet{catoni2007pac}, is defined with a hyperparameter (here named $c$). It appears to be a natural tool to design PAC-Bayesian algorithms.
We present this result in the simplified form suggested by~\citet{germain09b}.    
\begin{theorem}[\citet{catoni2007pac}] 
\label{thm:pacbayescatoni}
For any domain $\PS$ over  $X   \times   Y$, for  any set of hypotheses $\Hcal$,  any prior distribution $\prior$ over $\Hcal$, any $\delta \in (0,1]$, and any real number $c>0$,  with a probability at least $1 - \delta$ over the choice of $S \sim  (P_S)^m $, for every $\posterior$ on $\Hcal$, we have
\begin{equation*}
\RPS(G_{\posterior}) \ \leq\ \frac{c}{1 - e^{-c}}    \left[\RS(G_{\posterior})  +  \frac{\KL(\posterior\|\prior) + \ln  \frac{1}{\delta}}{m\times c}\right] .
\end{equation*}
\end{theorem}

The bound given by Theorem~\ref{thm:pacbayescatoni} has two interesting characteristics. First, choosing \mbox{$c = \tfrac{1}{\sqrt{m}}$}, the bound becomes consistent: it converges to {\small $1 \times \left[\RS(G_{\posterior})  +  0\right]$} as $m$~grows.
Second, as described in Section~\ref{sec:pbgd}, its minimization is closely related to the minimization problem associated with the {\small SVM} when $\posterior$ is an isotropic Gaussian over the space of linear classifiers \citep{germain2009pac}. Hence,  the value $c$ allows us to control the trade-off between the empirical risk $\RS(G_{\posterior})$ and the complexity term $\tfrac{1}{m}\,\KL(\posterior\|\prior) $.

\subsection{Supervised PAC-Bayesian Learning of Linear Classifiers} 
\label{sec:pbgd}

Let us consider $\Hcal$ as a set of  linear classifiers in a $d$-dimensional space. Each $h_{\wbf'}\in \Hcal$ is defined by a weight vector ${\wbf'}\in\mathbb{R}^d$:
$$h_{\wbf'}(\xb)\ \eqdef \ \sgn\LP\wbf'\cdot\xb\RP,$$ 
where $\,\cdot\,$ denotes the dot product.

\medskip
By restricting the prior and the posterior distributions over $\Hcal$ to be Gaussian distributions,  \citet{Langford02,AmbroladzePS06}; and~\citet{Parrado-Hernandez12} have specialized the PAC-Bayesian theory in order to bound the expected risk of any linear classifier $h_\wb\in\Hcal$.
More precisely, given a prior $\prior_{\mathbf{0}}$ and a posterior $\posterior_\wb$ defined as spherical Gaussians with identity covariance matrix respectively centered on vectors $\mathbf{0}$ and~$\wb$, for any  $h_{\wbf'}\in\Hcal$, we have
\begin{align*}
\prior_\mathbf{0}(h_{\wbf'})\, &\eqdef\, \LP \frac{1}{\sqrt{2\pii}} \RP^{ d}
  \exp\left({-\frac{1}{2}\|{\wbf'}\|^2}\right)\, ,\\
\mbox{ and } \quad 
    \posterior_\wb(h_{\wbf'})  \, &\eqdef\,   \LP \frac{1}{\sqrt{2\pii}} \RP^{ d}
  \exp\left({-\frac{1}{2}\|{\wbf'}-\wb\|^2}\right) \, .
 \end{align*}
 An interesting property of these Gaussian distributions is that the prediction of the \mbox{$\posterior_\wb$-weighted} majority vote $B_{\posterior_\wb}(\cdot)$ coincides with the one of the linear classifier $h_\wb(\cdot)$. Indeed, we have
 \begin{align*}
\forall\, \xb\in X,\ \forall\,\wb\in \Hcal, \quad 
h_\wb(\xb) \ &= \ B_{\posterior_\wb}(\xb) \\
 \ &= \ \sign\left[ \esp{h_{\wb'} \sim \posterior_\wb} h_{\wb'}(\xb) \right].
 \end{align*} 
 Moreover, the expected risk of the Gibbs classifier $G_{\posterior_\wb}$ on a domain $\PS$ is then given by
\begin{align*}
\RPS(G_{\posterior_\wb}) \  
 &= \esp{(\xb,y)\sim P_S} \  \esp{h_{\wbf'}\sim\posterior_\wb} \zoloss\big( h_{\wbf'}(\xb), y \big) \\    
  &= \esp{(\xb,y)\sim P_S} \   \esp{h_{\wbf'}\sim\posterior_\wb} \I\,\big( h_{\wbf'}(\xb) \neq  y \big) \\    
 &= \esp{(\xb,y)\sim P_S} \   \esp{h_{\wbf'}\sim\posterior_\wb}    \I\,\big(y \,\wb' \cdot \xb \leq 0\big)\\ 
 &=  \esp{(\xb,y)\sim P_S} \   \frac{1}{\sqrt{2\pii}}\  \int_{ \mathbb{R}^d} 
   \exp\left({-\frac{1}{2}\|{\wbf'}-\wb\|^2}\right) \, \I\,\big(y \,\wb' \cdot \xb \leq 0\big)\, d\, \wb' \\ 
 &= \esp{(\xb,y)\sim P_S}  \left[ 1 - \prob{t\sim\mathcal{N}(0,1)}\!\!\left(t\ \leq\  y\, \frac{\wb \cdot \xb}{\|\xb\|} \right) \right] \\
 &=  \esp{(\xb,y)\sim P_S} \Phi \left(  y\, \frac{\wb \cdot \xb}{\|\xb\|}  \right),
\end{align*}
where we defined
$$\Phi(a)  \ \eqdef  \ 
\frac{1}{2}  \left[1 - \Erf\left( \frac{a}{\sqrt{2}} \right) \right],$$
with $\Erf(\cdot)$ is the Gauss error function defined as
\begin{align}
\label{eq:erf}
\Erf\,(b)\ \eqdef\ \frac{2}{\sqrt{\pii}}\ \int_{0}^{b} \exp\left(-t^2\right) \text{d}t\,.
\end{align}
Finally, the KL-divergence between $\posterior_\wb$ and $\prior_\mathbf{0}$ becomes simply 
 $$\KL(\posterior_\wb \| \prior_\mathbf{0}) \ = \ \tfrac{1}{2}\| \wb \|^2\,. $$

\subsubsection{Objective Function and Gradient}
\label{sec:pbgd_objective}

Based on the specialization of the PAC-Bayesian theory to linear classifiers, \citet{germain2009pac} suggested minimizing a PAC-Bayesian bound on $\RPS(G_{\posterior_\wb})$. 
For sake of completeness, we provide here more mathematical details than in the original conference paper \citep{germain2009pac}. We will build on this PAC-Bayesian learning algorithm (for supervised leaning) in our domain adaptation work.

 Given a sample $S = \{(\xb^s_i, y^s_i)\}_{i=1}^m$ and a hyperparameter $C>0$,
the learning algorithm performs a gradient descent in order to find an optimal weight vector $\wb$ that minimizes
\begin{eqnarray}
\label{eq:prob_pbgd_primal}
F(\wb) 
& = & 
C m R_{S} (G_{\posterior_\wb})  + \KL(\posterior_\wb \| \prior_\mathbf{0}) \nonumber \\[1mm]
& = &
C  \displaystyle\sum_{i=1}^m \Phi \left(  y_i \frac{\wb \cdot \xb_i}{\|\xb_i\|}  \right)  
 + \frac{1}{2}\|\wb\|^2\,.
\end{eqnarray}It turns out that the optimal vector $\wb$ corresponds to the distribution $\posterior_\wb$ that minimizes the value of the bound on $\RPS(G_{\posterior_\wb})$ given by Theorem~\ref{thm:pacbayescatoni}, with the parameter $c$ of the theorem being the hyperparameter $C$ of the learning algorithm. It is important to point out that PAC-Bayesian theorems bound simultaneously $\RPS(G_{\posterior_\wb})$ \emph{for every $\posterior_\wb$ on $\Hcal$}. Therefore, one can ``freely'' explore the domain of objective function $F$ to choose a posterior distribution~$\posterior_\wb$ that gives, thanks to Theorem~\ref{thm:pacbayescatoni}, a bound valid with probability $1-\delta$.

The minimization of Equation~\eqref{eq:prob_pbgd_primal} by gradient descent corresponds to the learning  algorithm called {\small PBGD3} of \citet{germain2009pac}.
The gradient of $F(\wb)$ is given the vector~$\nabla F (\wb)$:
\begin{align*}
\nabla F(\wb) \ = \ C \sum_{i=1}^m 
\Phi' \LP y_i \frac{\wb\cdot\xb_i}{\|\xb_i\|} \RP  \frac{y_i\,\xb_i}{\|\xb_i\|}
+ \wb\,,
\end{align*}
where $\Phi'(a) = -\tfrac{1}{\sqrt{2\pii}} \exp\left(-\tfrac{1}{2} a^2\right)$ is the derivative of $\Phi(\cdot)$ at point $a$.

\medskip
Similarly to the SVM, the learning algorithm {\small PBGD3} realizes a trade-off between the empirical risk 
(expressed by the loss  $\Phi(\cdot)$) 
and the complexity of the learned linear classifier (expressed by the regularizer $\|\wb\|^2$).  
This similarity  increases when we use a kernel function, as described next.

\subsubsection{Using a kernel function}

The kernel trick allows to substitute inner products by a kernel
function $k:\R^d \times \R^d\rightarrow\R$ in
Equation~\eqref{eq:prob_pbgd_primal}. If $k$ is a Mercer kernel, it
implicitly represents a function $\phi:X\rightarrow\R^{d'}$ that maps
an example of $X$ into an arbitrary $d'$-dimensional space\footnote{We
consider here that the induced space is finite-dimensional.}, such that 
$$\forall(\xbf,\xbf')\in X^2,\quad k(\xb,\xb') \ =\ \phi(\xb)\cdot\phi(\xb')\,.$$ 
Then, a dual weight vector $\ab = (\alpha_1,\alpha_2, \ldots,\alpha_{m})\in\R^m$ encodes the linear classifier $\wb \in \R^{d'}$ as a linear combination of examples of $S$:
\begin{equation*}
\wb \ =\ 
\sum_{i=1}^m\alpha_i\, \phi(\xb_i)\,, 
\quad\mbox{ and thus } \quad
h_\wb(\xb) \ =\ 
\sgn\left[
\sum_{i=1}^m \alpha_i k(\xb_i, \xb)
\right].
\end{equation*}

By the representer theorem \citep{scholkopf-01}, the vector $\wb$ minimizing Equation~\eqref{eq:prob_pbgd_primal} can be recovered by finding the vector $\ab$ that minimizes
\begin{align} \label{eq:prob_pbgd_dual}
F(\ab) \ = \ 
C \sum_{i=1}^m  \Phi \left(  y_i \frac{\sum_{j=1}^{m} \alpha_j K_{i,j}}{ \sqrt{K_{i,i}}}  \right)+
\frac{1}{2} \sum_{i=1}^{m} \sum_{j=1}^{m} \alpha_i \alpha_j K_{i,j}\,,
\end{align}
where $K$ is the kernel matrix of size $m\times m$. That is, $K_{i,j} \eqdef \,k(\xb_i, \xb_j)\,.$
The gradient of $F(\ab)$ is simply given the vector $\nabla F (\ab) = (\alpha_1',\alpha_2', \ldots\alpha_{m}')$, with 
\begin{align*} 
 \alpha'_\# \ = \ 
 C \sum_{i=1}^m  \Phi \left(  y_i \frac{\sum_{j=1}^{m} \alpha_j K_{i,j}}{ \sqrt{K_{i,i}} }  \right)   \frac{y_i \,K_{i,\#}}{ \sqrt{K_{i,i}} } 
 + \sum_{j=1}^{m} \alpha_i K_{i,\#}\,,
 \quad \mbox{ for } \# \in \{1,2,\ldots,m\,\}\,.
\end{align*}

\subsubsection{Improving the Algorithm Using a Convex Objective }
\label{section:pbgd3_convex}
An annoying drawback of {\small PBGD3} is that the objective function is non-convex and the gradient descent implementation needs many random restarts. In fact, we  made extensive empirical experiments after the ones described by~\citet{germain2009pac} and saw that {\small PBGD3}  achieves an equivalent accuracy (and at a fraction of the running time) by replacing the  loss function $\Phi(\cdot)$ of Equations~\eqref{eq:prob_pbgd_primal} and~\eqref{eq:prob_pbgd_dual} by its convex relaxation, which is
\begin{align*}
\Phic(a) \ &\eqdef\ 
\max \left\{\Phi(a),\, \frac{1}{2} - \frac{a}{\sqrt{2\pii}} \right\}\\
\ &=
\,\left\{
\begin{array}{ll}
 \displaystyle \frac{1}{2} - \frac{a}{\sqrt{2\pii}} & \mbox{if $a\leq 0$},\\
 \Phi(a) & \mbox{otherwise}.
\end{array}\right.
\end{align*}
The derivative of $\Phic(\cdot)$ at point $a$ is then 
$\Phic'(a) = \frac{-1}{\sqrt{2\pii}}$ if $a<0$, and $\Phi'(a)$ otherwise.
Note that Figure~\ref{fig:phi3} in Section~\ref{sec:pbda} illustrates the functions $\Phi(\cdot)$ and $\Phic(\cdot)$\,.\\ 

\noindent In the following we present our contributions on PAC-Bayesian domain adaptation.


\section{PAC-Bayesian Theorems for Domain Adaptation}
\label{sec:our_bound}

The originality of our contribution is to theoretically design a domain adaptation framework for PAC-Bayesian approach.
In Section \ref{sec:distance}, we propose a domain comparison pseudometric suitable in this context.
We then derive PAC-Bayesian domain adaptation bounds in Section~\ref{sec:pacbayesdabound}, that improves the result proposed in \citet{pbda}. Finally, note that in Section~\ref{sec:pbda} we see that using the previous approach in a domain adaptation way is a relevant strategy:  we specialize our result to linear classifiers.

\subsection{A Domain Divergence for PAC-Bayesian Analysis}
\label{sec:distance}

In the following, while the domain adaptation bounds presented in Section~\ref{sec:notations} focus on a single classifier, we first define a $\posterior$-average disagreement measure to compare the marginals. Then, this leads us to derive our domain adaptation bound suitable for the PAC-Bayesian approach. 

As discussed in Section \ref{sec:necessity_dist}, the derivation of generalization ability in domain adaptation critically needs a  divergence measure between the source and target marginals.

\subsubsection{Designing the Divergence}
We define a \emph{domain disagreement pseudometric}\footnote{A pseudometric $d$ is a metric for which the property  \mbox{$d(x,y)=0 \Leftrightarrow x=y$} is relaxed to \mbox{$d(x,y)=0 \Leftarrow x=y$}.} to measure the structural difference between domain marginals in terms of posterior distribution $\posterior$ over $\Hcal$.
Since we are interested in learning a $\posterior$-weighted majority vote $B_\posterior$ leading to good generalization guarantees, we propose to follow the idea behind the $C$-bound presented in Equation~\eqref{eq:C-bound}: 
given $\PS$, $\PT$, and  $\posterior$, if $\RPS(G_\posterior)$ and $\RPT(G_\posterior)$ are similar, then $\RPS(B_\posterior)$ and  $\RPT(B_\posterior)$ are similar when $\esp{(h,h')\sim \posterior^2 }     \RDS(h,h')$ and $\esp{(h,h')\sim \posterior^2 }      \RDT(h, h')$ are also similar.
Thus, the domains $\PS$ and $\PT$  are close according to  $\posterior$ if the divergence between $\esp{(h,h')\sim \posterior^2 }     \RDS(h,h')$ and $\esp{(h,h')\sim \posterior^2 }      \RDT(h, h')$ tends to be low.
Our pseudometric is  defined as follows.
\begin{definition}
\label{def:disagreement}
Let  $\Hcal$ be a hypothesis class. For any marginal distributions $\DS$ and $\DT$ over $X$, any distribution $\posterior$ on $\Hcal$, the domain disagreement $\des(\DS,\DT)$ between $\DS$ and $\DT$ is defined by
\begin{align*}  
\des(\DS,\DT)  \ &\eqdef \ \left| \esp{(h,h')\sim \posterior^2 }    \Big[ \RDT(h,h')  - \RDS(h, h') \Big] \right|\\
 &= \  
\Big|\,\RDT(G_\posterior,G_\posterior)  - \RDS(G_\posterior,G_\posterior) \,\Big|
\,.
\end{align*}
\end{definition}
Note that $\des(\cdot, \cdot)$ is symmetric and fulfills the triangle inequality.

\subsubsection{Comparison of the {\small $\hdh$}-divergence and our domain disagreement} 
While  the {\small $\hdh$}-divergence of Theorem~\ref{theo:BenDavid} is difficult to jointly optimize with the empirical source error, our empirical disagreement measure is easier to manipulate:  we simply need to compute the $\posterior$-average of the classifiers disagreement instead of finding the pair of classifiers that maximizes the disagreement. 
Indeed, $\des(\cdot,\cdot)$ depends on the majority vote, which suggests that we can  directly minimize it via the empirical $\des(S,T)$ and the KL-divergence. This can be done without instance reweighing, space representation changing or family of classifiers modification.
On the contrary,~$\tfrac{1}{2}d_{\hdh}(\cdot,\cdot)$ is a supremum over all $h \in \Hcal$ and hence, does not depend on the $h$ on which the risk is considered.
Moreover, $\des(\cdot,\cdot)$ (the $\posterior$-average) is lower than the $\tfrac{1}{2}d_{\hdh}(\cdot,\cdot)$ (the worst case).
Indeed, for every $\Hcal$ and $\posterior$ over $\Hcal$, we~have
\begin{align*}
\tfrac{1}{2}\, d_{\hdh}(\DS,\DT)  \ & =\ \sup_{\substack{(h,h')\in\mathcal{H}^2}} \left|\RDT(h,h')-\RDS(h,h')\right| \\
&\geq\  \esp{(h,h')\sim\posterior^2} \left| \RDT(h,h')-\RDS(h,h')\right|\\
&\geq\   \des(\DS,\DT)\,.
\end{align*}

\subsubsection{PAC-Bayesian bounds for our domain disagreement}
\label{section:PB-dis}

The following theorems show that $\des(\DS,\DT)$ can be bounded in terms of the classical PAC-Bayesian quantities: the empirical disagreement $\des(S,T)$ estimated on the source and target samples, and the KL-divergence between the prior and posterior distribution on~$\Hcal$.

\medskip

For the sake of simplicity, let first suppose that $m = m'$, {\it i.e.}, the size of $S$ and $T$ are~equal. 
Here is a ``Seeger's type'' PAC-Bayesian bound for our domain disagreement~$\des$.
\begin{theorem} \label{thm:bound_dis_kl}
 For any distributions $\DS$ and $\DT$ over $X$, any set of hypotheses $\Hcal$,  and any prior distribution $\prior$ over $\Hcal$, any $\delta \in (0,1]$, with a probability at least $1 - \delta$ over the choice of $S \times  T \sim  (D_S \times  D_T)^m$, for every $\posterior$ on $\Hcal$, we have
 \begin{align*}
 \kl \left(  \frac{\des(S,T) +1}{2} \Bigg\|\frac{\des(D_S,D_T) +1}{2}  \right)\ \leq\     \frac{1}{m} \left[2\,\KL(\posterior\|\prior) +  \ln  \frac{2\sqrt{m}}{\delta}\right].
 \end{align*}
 \end{theorem}
\begin{proof}
Deferred to Appendix \ref{app:seeger}.
\end{proof}

Here is a ``McAllester's type'' PAC-Bayesian bound for our domain disagreement~$\des$ obtained straightforwardly from Theorem~\ref{thm:bound_dis_kl}.
\begin{corollary} \label{cor:bound_dis_kl_pinsker}
 For any distributions $\DS$ and $\DT$ over $X$, any set of hypotheses $\Hcal$,  and any prior distribution $\prior$ over $\Hcal$, any $\delta \in (0,1]$, with a probability at least $1 - \delta$ over the choice of $S \times  T \sim  (D_S \times  D_T)^m$, for every $\posterior$ on $\Hcal$, we have
 \begin{align*}
\Big| \des(D_S,D_T) - \des(S,T)  \Big|\ \leq\    2\times\sqrt{ \frac{1}{2m} \left[2\,\KL(\posterior\|\prior) +  \ln  \frac{2\sqrt{m}}{\delta}\right]}.
 \end{align*}
 \end{corollary}
\begin{proof}
The result is obtained by using Pinsker's inequality (Equation~\eqref{eq:pinsker}) on Theorem~\ref{thm:bound_dis_kl}.
\end{proof}

Here is a ``Catoni's type'' PAC-Bayesian bound which helps us to derive a domain adaptation algorithm in the following.
\begin{theorem} \label{thm:bound_dis_rho}
 For any distributions $\DS$ and $\DT$ over $X$, any set of hypotheses $\Hcal$,  any prior distribution $\prior$ over $\Hcal$, any $\delta \in (0,1]$, and any real number $\alpha > 0$,  with a probability at least $1 - \delta$ over the choice of $S \times  T \sim  (D_S \times  D_T)^m$, for every $\posterior$ on $\Hcal$, we have
 \begin{align*}
\des(\DS,\DT)\ \leq\   \frac{2\alpha }{1 -e^{-2\alpha}}  \left[ \des(S,T)  +  \frac{2\,\KL(\posterior\|\prior)  +  \ln  \frac{2}{\delta}}{m\times\alpha} + 1\right] - 1\,.
\end{align*}
\end{theorem}
\begin{proof} 
Deferred to Appendix \ref{app:catoni}.
\end{proof}
Similarly to the empirical risk bound of~\citet{catoni2007pac} shown by Theorem~\ref{thm:pacbayescatoni}, the above domain disagreement bound is consistent if one puts $\alpha = \frac{1}{2\sqrt{m}}$.
Indeed, it converges to $1 \times \left[\des(S,T) + 0 + 1\right] - 1$ as $m$ grows.

\medskip 
The last result of this section tackles the situation where  $m  \neq m'$, {\it i.e.}, the sizes of~$S$ and~$T$ are different.

\begin{theorem}\label{thm:bound_dis_rho_allester}
For any marginal distributions $\DS$ and $\DT$ over $X$,  any set of hypotheses  $\Hcal$, any prior distribution $\prior$ over $\Hcal$, any $\delta\in(0,1]$,  with a probability at least $1-\delta$ over the choice of $S\sim (\DS)^m$ and $T\sim (\DT)^{m'}$, for every $\posterior$ over $\Hcal$, we have
\begin{equation*}
\bigg|\,\des(\DS, \DT)  -  \des(S,T)\, \bigg|
\  \leq \  
 \sqrt{ \frac{ 2\,\KL(\posterior\|\prior) + \ln\!\frac{4\sqrt{m}}{\delta}}{2m}}
 + 
  \sqrt{ \frac{ 2\,\KL(\posterior\|\prior) + \ln\!\frac{4\sqrt{m'}}{\delta}}{2m'}}\,.
\end{equation*}
\end{theorem}

\begin{proof}Deferred to Appendix \ref{app:mcallester}.
\end{proof}
 Note that Theorem~\ref{thm:bound_dis_rho_allester} is very similar to the result of Corollary~\ref{cor:bound_dis_kl_pinsker}. 
In fact, in the particular case $m=m'$, Theorem~\ref{thm:bound_dis_rho_allester} differs from Corollary~\ref{cor:bound_dis_kl_pinsker} only by the $4\sqrt{m}$ term inside the logarithm, instead of $2\sqrt{m}$.

\subsection{PAC-Bayesian Theorems for Domain Adaptation}
\label{sec:pacbayesdabound}

We now derive our main result in the following theorem: a domain adaptation bound relevant in a PAC-Bayesian setting. 

\subsubsection{A domain adaptation bound for the stochastic Gibbs classifier}

Theorem~\ref{thm:pacbayesdabound} below relies on the domain disagreement of Definition~\ref{def:disagreement}, and also on  \emph{expected  joint error} of Equation~\eqref{eq:eP}.

\begin{theorem}
\label{thm:pacbayesdabound}
Let ${\cal H}$ be a hypothesis class. We have 
\begin{align*}
\nonumber \forall \posterior&\mbox{ on }\Hcal,\  \RPT(G_\posterior)\ \leq \  \RPS(G_\posterior) +  \frac{1}{2}\des(\DS,\DT) + \lambda_\posterior\,, 
\end{align*}
where $\lambda_\rho$  
is the deviation between the expected joint errors of $G_\posterior$ on the target and source domains:
 \begin{eqnarray} \label{eq:lambda_rho}
 \nonumber
 \lambda_\posterior\!&\eqdef& \left|\esp{(h,h')\sim\posterior^2}\!\!\left[\esp{(\xbf,y) \sim \PT}\!\!\!\! \zoloss\big( h(\xb), y \big)\zoloss\big( h'(\xb), y \big) -\!\! \esp{(\xbf,y) \sim \PS}\!\!\!\! \zoloss\big( h(\xb), y \big) \zoloss\big( h'(\xb), y \big)\right]\right|\\
 &=&
 \Big|\, \ePT(G_\posterior, G_\posterior) - \ePS(G_\posterior,G_\posterior) \,\Big|\,.
\end{eqnarray}
\end{theorem}
\smallskip

\begin{proof}
First, notice that for any distribution $P$ on $X\times Y$ (and corresponding marginal distribution $D$ on $X$), we have
\begin{equation} \label{eq:rde}
R_P(G_\posterior) \ = \ \frac{1}{2} \RD(G_\posterior,\GQ) + \eP(G_\posterior,G_\posterior)\,,
\end{equation}
as
\begin{eqnarray*}
2\, R_P(G_\posterior)
&=& 
\esp{(h,h') \sim\posterior^2}\esp{(\xbf,y) \sim P} 
\Big[ \zoloss\big( h(\xb), y \big) + \zoloss\big(h'(\xb), y \big) \Big] \\
&=& 
\esp{(h,h') \sim\posterior^2}\esp{(\xbf,y) \sim P} 
\Big[ 1\times\zoloss\big( h(\xb), h'(\xb) \big) + 2\times\zoloss\big(h(\xb), y \big)\, \zoloss\big(h'(\xb), y \big) \Big]\\[1.5mm]
&=& 
\RD(G_\posterior,\GQ) + 2\times \eP(G_\posterior,G_\posterior)\,.
\end{eqnarray*}
Therefore,
\begin{eqnarray*}
\nonumber \RPT(G_\posterior)-\RPS(G_\posterior)\!\!
& = &
\nonumber \frac{1}{2} \Big(\RDT(G_\posterior,\GQ)-\RDS(G_\posterior,\GQ)\Big) \!+\!\Big(\ePT(G_\posterior,G_\posterior)-\ePS(G_\posterior,G_\posterior)\Big) \\
&\leq&
\nonumber \frac{1}{2} \Big|\RDT(G_\posterior,\GQ)-\RDS(G_\posterior,\GQ)\Big| +\Big|\ePT(G_\posterior, G_\posterior)-\ePS(G_\posterior, G_\posterior)\Big|  \\
&=&
\frac{1}{2} \des(\DS,\DT)  + \lambda_\posterior \,.
\end{eqnarray*}
\end{proof}

Our bound is, in general, incomparable with the ones of Theorems~\ref{theo:BenDavid} and~\ref{theo:Mansour}.
It can be seen as a trade-off  between different quantities.
The terms $\RPS(G_\posterior)$ and  $\des(\DS,\DT)$ are similar to the first two terms of the domain adaptation bound of \citet{BenDavid-MLJ2010} (Equation \eqref{eq:da}):   $\RPS(G_\posterior)$ is the $\posterior$-average risk over $\Hcal$ on the source domain,  and $\des(\DT,DS)$ measures the $\posterior$-average disagreement between the marginals but is specific to the current $\posterior$.
The other term $\lambda_\posterior$ measures the deviation between the expected joint target and source errors of $\GQ$.
According to this theory, a good domain adaptation is possible if this deviation is low. 
However, since we suppose that we do not have any label in the target sample, we cannot control or estimate it.
In practice, we suppose that $\lambda_\posterior$ is low and we neglect
it. In other words, we assume that the labeling information between
the two domains is related and that considering only the marginal
agreement and the source labels is sufficient to  find a good majority
vote. 
Another important point comes from the fact that this bound is not
degenerated when the source and target distributions are the same or
close, see Section \ref{sec:comparaison} for a discussion on this point.

In the next section, we provide three PAC-Bayesian theorems that justifies the empirical optimization of the bound of Theorem~\ref{thm:pacbayesdabound}.

\subsubsection{PAC-Bayesian theorems for domain adaptation}
\label{sec:3PB}
Finally, our Theorem~\ref{thm:pacbayesdabound} leads to a PAC-Bayesian bound based on both the empirical source error of the Gibbs classifier  and the empirical domain disagreement pseudometric estimated on a source and target samples.

From the preceding ``Seeger's type'' results, one can then obtain the following PAC-Bayesian domain adaptation bound.
\begin{theorem}
 \label{theo:pacbayesdabound_bis_seeger}
 For any domains $\PS$ and $\PT$ (respectively with marginals $\DS$ and $\DT$) over $X \times  Y$, any set of hypotheses  $\Hcal$,  any prior distribution $\prior$ over $\Hcal$,  and any $\delta \in (0,1]$, with a probability at least $1-\delta$ over the choice of $S \times  T  \sim (\PS \times  D_T)^m $, we have
\begin{align*}
\RPT(G_\posterior) 
\ \leq  \ \sup \mathcal{R}_\posterior  + \tfrac{1}{2} \sup \mathcal{D}_\posterior  + \lambda_\posterior\,,
\end{align*}
where $\lambda_\rho$ is defined by Equation~\eqref{eq:lambda_rho}, 
and
\begin{align*}
\mathcal{R}_\posterior   \ \eqdef\  &\left\lbrace r  :  
\kl\big(\RS(G_{\posterior})\big\| r\big) \leq   \tfrac{1}{m} \left[\KL(\posterior\|\prior)+ \ln \tfrac{4\sqrt{m}}{\delta}\right]
 \right\rbrace ,\\
 \mathcal{D}_\posterior  \ \eqdef \  & \left\lbrace d  :  
 \kl\big(\tfrac{\des(S,T) +1}{2}\big\| \tfrac{d+1}{2}\big)  \leq    \tfrac{1}{m}  \left[2\,\KL(\posterior\|\prior) +  \ln \tfrac{4\sqrt{m}}{\delta}\right]
   \right\rbrace  .
\end{align*}
\end{theorem}
\begin{proof}
The result is obtained by inserting Theorems~\ref{thm:pacbayesseeger} and~\ref{thm:bound_dis_kl} (with $\delta := \frac{\delta}{2}$) in Theorem~\ref{thm:pacbayesdabound}.
\end{proof}

The following bound is based on Catoni's approach and corresponds to the one from which we derive---in Section~\ref{sec:pbda}---our algorithm for PAC-Bayesian domain adaptation.
\begin{theorem}
 \label{theo:pacbayesdabound_catoni_bis}
 For any domains $\PS$ and $\PT$ (resp. with marginals $\DS$ and $\DT$) over $X \times  Y$, any set of hypotheses  $\Hcal$,  any prior distribution $\prior$ over $\Hcal$, any $\delta \in (0,1]$, any real numbers $\alpha > 0$ and $c > 0$,  with a probability at least $1-\delta$ over the choice of $S \times  T  \sim (\PS \times  D_T)^m $, for every posterior distribution $\posterior$ on $\Hcal$, we have
 \begin{align*} 
\RPT(G_\posterior)  
\ \leq\  c'\, \RS(G_\posterior)  +  \alpha'\, \tfrac{1}{2} \des(S,T) + 
  \left( \frac{c'}{c} + \frac{\alpha'}{\alpha} \right)  \frac{\KL(\posterior\|\prior)+\ln\frac{3}{\delta}}{m}
   + \lambda_\posterior + \tfrac{1}{2} (\alpha'\!-\! 1)
   \,,
 \end{align*}
where 
$\lambda_\rho$ is defined by Equation~\eqref{eq:lambda_rho}, 
and where \,
$\displaystyle c'\eqdef\frac{c}{1 -e^{-c}}$, \, and \, $\displaystyle \alpha'\eqdef \frac{2\alpha}{1 -e^{-2\alpha}}$\,.
\end{theorem}
\begin{proof}
In Theorem~\ref{thm:pacbayesdabound}, we replace $\RS(G_\posterior)$ and $\des(S,T)$ by their upper bound, obtained from Theorem~\ref{thm:pacbayescatoni} and Theorem~\ref{thm:bound_dis_rho}, with $\delta$ chosen respectively as $\frac{\delta}{3}$ and $\frac{2\delta}{3}$. In the latter case, we use 
\begin{align*}2\,\KL(\posterior\|\prior) + \ln\tfrac{2}{2\delta/3}
\ &=\ 2\,\KL(\posterior\|\prior) +\ln\tfrac{3}{\delta} \\
\ &<\  2\left( \KL(\posterior\|\prior) +\ln\tfrac{3}{\delta} \right)\,.
\end{align*}
\end{proof}

We now present a result based on the McAllester bound, which allows us to easily deal with different sizes of samples.
\begin{theorem}
 \label{theo:pacbayesdabound_bis}
 For any domains $\PS$ and $\PT$ (respectively with marginals $\DS$ and $\DT$) over $X\times Y$, and for any set $\Hcal$ of hypotheses, for any prior distribution $\prior$ over $\Hcal$, any $\delta\in(0,1]$,  with a probability at least $1-\delta$ over the choice of $S_1\sim (\PS)^{m_1}$, $S_2\sim (D_S)^{m_2}$, and $T\sim (D_T)^{m'}$, for every $\posterior$ over $\Hcal$, we have
\begin{align*}
\RPT(G_\posterior) 
\, \leq \   
& R_{S_1}(G_\posterior)  + \tfrac{1}{2}\des(S_2,T) + \lambda_\posterior \\
&+ 
\sqrt{\frac{\KL(\posterior\|\prior)  +  \ln  \tfrac{4\sqrt{m_1}}{\delta}}{2m_1}}
+ 
 \sqrt{\frac{2\,\KL(\posterior\|\prior)  +  \ln  \tfrac{8\sqrt{m_2}}{\delta}}{8m_2}}
+ 
\sqrt{\frac{2\,\KL(\posterior\|\prior)  +  \ln  \tfrac{8\sqrt{m'}}{\delta}}{8m'}}\,,
       \end{align*}
where 
$\lambda_\rho$ is defined by Equation~\eqref{eq:lambda_rho}.
\end{theorem}
\begin{proof}
We insert Theorems~\ref{thm:pacbayesallester} and  \ref{thm:bound_dis_rho_allester} (with $\delta := \frac{\delta}{2}$) in Theorem~\ref{thm:pacbayesdabound}.
\end{proof}

Under the assumption that the domains are somehow related in terms of labeling agreement on $\PS$ and $\PT$ (for every distribution $\posterior$ over $\Hcal$), {\it i.e.}, a low  $\des(\DS,\DT)$ implies a negligible $\lambda_\posterior$, 
a natural solution for a PAC-Bayesian domain adaptation algorithm without target label is to minimize the bound of Theorem~\ref{theo:pacbayesdabound_catoni_bis} by disregarding 
$\lambda_\posterior$.
Notice that a major advantage of our domain adaptation bound is that we can jointly optimize the risk and the divergence with a theoretical justification.


\section{PAC-Bayesian Domain Adaptation Learning of Linear Classifiers} 
\label{sec:dapbgd}
\label{sec:pbda}

In this section,  we design a learning algorithm for domain adaptation inspired by the PAC-Bayesian learning algorithm of~\citet{germain2009pac}. That is, we adopt the specialization of the PAC-Bayesian theory to linear classifiers described in Section~\ref{sec:pbgd}. 
Note that the code of our algorithm is available on-line.\footnote{See \mbox{\url{http://graal.ift.ulaval.ca/pbda}}.}

\subsection{Minimizing the PAC-Bayesian Domain Adaptation Bound}

Let us consider a prior $\prior_\mathbf{0}$ and a posterior $\posterior_\wb$ that are spherical Gaussian distributions over a space of linear classifiers, exactly as defined in Section~\ref{sec:pbgd}.

Given a source sample $S = \{(\xb^s_i, y^s_i)\}_{i=1}^m$ and a target sample $T = \{(\xb^t_i)\}_{i=1}^m$, 
we focus on the minimization of the bound given by Theorem~\ref{theo:pacbayesdabound_catoni_bis}. We work under the assumption that the term $\lambda_{\posterior_\wb}$ of the bound is negligible.
Thus, the posterior distribution $\posterior_\wb$ that minimizes the bound on $\RT(G_{\posterior_\wb})$ is the same that minimizes
\begin{equation} \label{eq:probpbda1}
C\,m\, \RS(G_{\posterior_\wb})  +  A \,m\, \desw(S,T) +  \KL(\posterior_\wb \| \prior_\mathbf{0})\,.
\end{equation}
The values $A>0$ and $C>0$ are hyperparameters of the algorithm. Note that the constants $\alpha$ and $c$ of Theorem~\ref{theo:pacbayesdabound_catoni_bis} can be recovered from any $A$ and $C$. 

\subsubsection{Domain Disagreement of Linear Classifiers}
\label{sec:pbda_dislinear}

We know from Equation~\eqref{eq:prob_pbgd_primal} how to compute the terms 
$\RS(G_{\posterior_\wb})$ and $\KL(\posterior_\wb \| \prior_\mathbf{0})$ of Equation~\eqref{eq:probpbda1}.
Let us now derive the value of $\desw(S,T)$,  \emph{i.e.}, the empirical domain disagreement between $S$ and $T$ of a distribution $\posterior_\wb$ over linear classifiers.

\smallskip
\noindent
First, for any marginal $D$, we obtain
\begin{align*}
R_D(G_{\posterior_\wb}, G_{\posterior_\wb}) \ 
=&\  \esp{\xbf\sim D}\  \esp{(h,h')\sim \posterior_\wb^2} \zoloss\big( h(\xb), h'(\xb) \big) \\
=&\  \esp{\xbf\sim D}\  \esp{(h,h')\sim \posterior_\wb^2} \I [h(\xb)\neq h'(\xb)]\\
=& \ \esp{\xbf\sim D}\  \esp{(h,h')\sim \posterior_\wb^2} 
\Big(\I [h(\xb)=1] \, \I [h'(\xb)=-1] + \I [h(\xb)=-1] \, \I [h'(\xb)=1]\Big)\\
=& \ 2\  \esp{\xbf\sim D}\  \esp{(h,h')\sim \posterior_\wb^2} \I [h(\xb)=1] \, \I [h'(\xb)=-1]\\
=& \ 2 \ \esp{\xbf\sim D}\  \esp{h\sim \posterior_\wb} \I [h(\xb)=1] \,  \esp{h'\sim \posterior_\wb}\I [h'(\xb)=-1]\\
=& \ 2\  \esp{\xbf\sim D} \Phi \left( \frac{\wb \cdot \xb}{\|\xb\|}  \right) \  \Phi \left(  - \frac{\wb \cdot \xb}{\|\xb\|}  \right). 
\end{align*}
Thus,
\begin{eqnarray*}
\desw(S,T)  &=& \Big| \,  \RS(G_{\posterior_\wb},G_{\posterior_\wb})-      \RT(G_{\posterior_\wb},G_{\posterior_\wb})\,\Big| \\
&=& 
\left| \frac{1}{m} \sum_{i=1}^m \Phidis \left(  \frac{\wb \cdot \xb^s_i}{\|\xb^s_i\|}  \right)    -  \frac{1}{m}\sum_{i=1}^m  \Phidis \left(  \frac{\wb \cdot \xb^t_i}{\|\xb^t_i\|}  \right)  \right|,
\end{eqnarray*}
where
\begin{equation}
\Phidis(a)\ \eqdef \ 2\,\Phi(a)\,\Phi(-a)\,.
\end{equation}

\subsubsection{Objective Function and Gradient}

From the results of Sections~\ref{sec:pbgd_objective} and~\ref{sec:pbda_dislinear}, we obtain that Equation~\eqref{eq:probpbda1} equals to
\begin{equation*} \label{eq:probpbda2}
C \sum_{i=1}^m  \Phi \left(  y^s_i \frac{\wb \cdot \xb^s_i}{\|\xb^s_i\|}  \right)   +  
A  \left| \sum_{i=1}^m  \left[ \Phidis \left(  \frac{\wb \cdot \xb^s_i}{\|\xb^s_i\|}  \right)    -  \Phidis \left(  \frac{\wb \cdot \xb^t_i}{\|\xb^t_i\|}  \right) \right] \right|
  +   \frac{1}{2}\|\wb\|^2\,,
\end{equation*}
\noindent
which is highly non-convex.
To make the optimization problem more tractable, we replace the loss function $\Phi(\cdot)$ by its convex relaxation $\Phic(\cdot)$ (as in Section~\ref{section:pbgd3_convex}) and minimize the resulting cost function by gradient descent. 
Even if this optimization task is still not convex ($\Phidis(\cdot)$ is quasiconcave), our empirical study shows no need to perform many restarts to find a suitable solution.\footnote{We observe empirically that a good strategy is to first find the vector $\wb$ minimizing the convex problem of {\small PBGD3} described in Section~\ref{section:pbgd3_convex}, and then use this $\wb$ as a starting point for the gradient descent of \PBDA.}  

\smallskip
We name this domain adaptation algorithm \PBDA.
To sum up, given a source sample $S = \{(\xb^s_i, y^s_i)\}_{i=1}^m$, a target sample $T = \{(\xb^t_i)\}_{i=1}^m$, and hyperparameters $A$ and $C$, the algorithm \PBDA performs gradient descent to minimize the following objective function:
\begin{equation} \label{eq:probpbda3}
G(\wb) = C \sum_{i=1}^m  \Phic \left(  y^s_i \frac{\wb \cdot \xb^s_i}{\|\xb^s_i\|}  \right)   +  
A  \left| \sum_{i=1}^m \left[ \Phidis \left(  \frac{\wb \cdot \xb^s_i}{\|\xb^s_i\|}  \right)    -  \Phidis \left(  \frac{\wb \cdot \xb^t_i}{\|\xb^t_i\|}  \right)  \right]  \right|
  +   \frac{1}{2}\|\wb\|^2\,,
\end{equation}
where
\begin{eqnarray*}
\Phi(a) & \eqdef &  \frac{1}{2}  \left[1 - \Erf\left( \frac{a}{\sqrt{2}} \right) \right]\,,\\
\Phic(a) & \eqdef & 
\max \left\{\Phi(a),\, \frac{1}{2} - \frac{a}{\sqrt{2\pii}} \right\}\,,\\
\Phidis(a)& \eqdef & 2\times\Phi(a)\times\Phi(-a) \,,
\end{eqnarray*}
with $\Erf(\cdot)$  the Gauss error function defined in Equation~\eqref{eq:erf}.
Figure~\ref{fig:phi3} illustrates these three functions.
\begin{figure} 
\centering
\includegraphics[width=0.7\columnwidth]{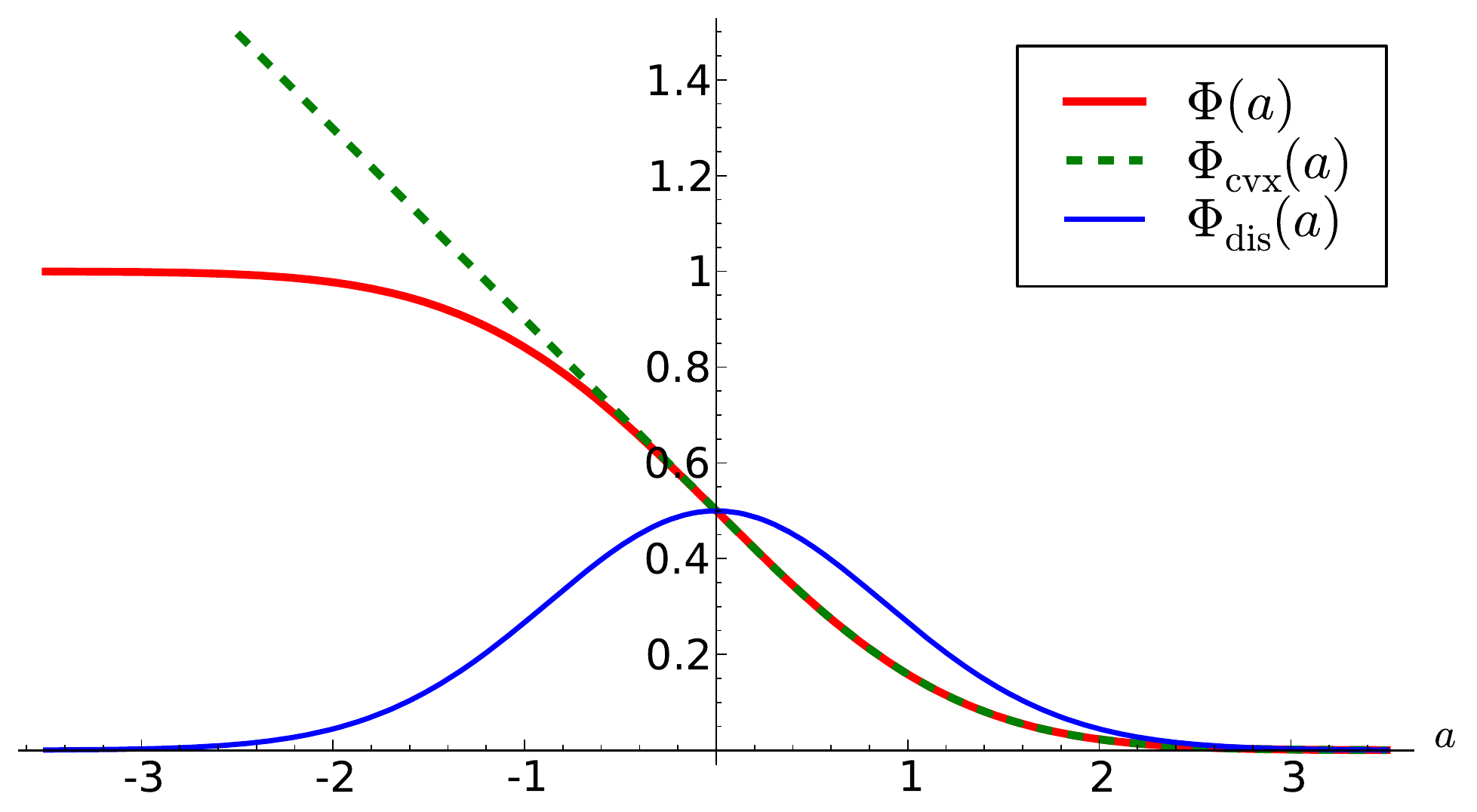}
\caption{Behavior of functions $\Phi(\cdot)$, $\Phic(\cdot)$ and $\Phidis(\cdot)$.}
\label{fig:phi3} 
\end{figure}

\smallskip
\noindent
The gradient $\nabla G (\wb)$ of the Equation~\eqref{eq:probpbda3} is then given by
\begin{align*}
\nabla G (\wb)\ =\ &C \sum_{i=1}^m 
\Phic' \LP\frac{y^s_i\wb\cdot\xb^s_i}{\|\xb^s_i\|} \RP  \frac{y^s_i\xb^s_i}{\|\xb^s_i\|} +\wb\\
&+ s\!\times\! A\left(\sum_{i=1}^m \left[\Phi'_{\rm dis} \LP\frac{\wb\cdot\xb^t_i}{\|\xb^t_i\|} \RP  \frac{\xb^t_i}{\|\xb^t_i\|} -
\Phi'_{\rm dis} \LP\frac{\wb\cdot\xb^s_i}{\|\xb^s_i\|} \RP  \frac{\xb^s_i}{\|\xb^s_i\|} \right]\right),
\end{align*}
where $\Phic'(a)$ and $\Phidis'(a)$ are respectively the derivatives of functions $\Phic(\cdot)$ and $\Phidis(\cdot)$ evaluated at point $a$, and 
$$
s = \sgn \left(\ \displaystyle\sum_{i=1}^m  \left[ \Phidis \left(  \frac{\wb \cdot \xb^s_i}{\|\xb^s_i\|}  \right)    -  \Phidis \left(   \displaystyle\frac{\wb \cdot \xb^t_i}{\|\xb^t_i\|}  \right)  \right] \right)\,.
$$
We extend these equations to kernels in the following subsection.

\subsubsection{Using a Kernel Function}

The kernel trick allows us to work with
dual weight vector $\ab\in \R^{2m}$ that is a linear classifier in an augmented space.  Given a kernel $k:\R^d \times \R^d\rightarrow\R$, we have
\begin{equation*}
h_\wb(\xb) \ =\ 
\sgn\left[
\sum_{i=1}^m \alpha_i k(\xb^s_i, \xb) +  \sum_{i=1}^m \alpha_{i+m} k(\xb^t_i, \xb)
\right].
\end{equation*}
Let us denote $K$ the kernel matrix of size $2m\times 2m$ such as
$K_{i,j} \eqdef k(\xb_i, \xb_j)\,,$
where
$$\xb_\# \, =\, 
\begin{cases}
\xb^s_\# & \mbox{if } \# \leq m \\
\xb^t_{\#-m} &\mbox{otherwise.} 
\end{cases} 
$$
In that case, the objective function of Equation~\eqref{eq:probpbda3} is rewritten in terms of the vector $\ab = (\alpha_1,\alpha_2, \ldots\alpha_{2m})$ as
\begin{align*} 
G(\ab) =  
 C &\sum_{i=1}^m  \Phic \left(  y^s_i \frac{\sum_{j=1}^{2m} \alpha_j K_{i,j}}{ \sqrt{K_{i,i}} }  \right) \\ 
 &{}+ 
A  \left| \sum_{i=1}^m  \left[\Phidis \left(  \frac{\sum_{j=1}^{2m} \alpha_j K_{i,j}}{ \sqrt{K_{i,i}} }  \right)    -  \Phidis \left(  \frac{\sum_{j=1}^{2m} \alpha_j K_{i+m,j}}{ \sqrt{K_{i+m,i+m}} }  \right) \right] \right| 
+ \frac{1}{2} \sum_{i=1}^{2m} \sum_{j=1}^{2m} \alpha_i \alpha_j K_{i,j} \,.
\end{align*}
The gradient of the latter equation is given by the vector $\nabla G (\ab) = (\alpha_1',\alpha_2', \ldots\alpha_{2m}')$, with
\begin{align*} 
\alpha'_\#\ =\  
 &C  \sum_{i=1}^m  \Phic' \left(  y^s_i \frac{\sum_{j=1}^{2m} \alpha_j K_{i,j}}{ \sqrt{K_{i,i}} }  \right) 
  \tfrac{y^s_i \,K_{i,\#}}{ \sqrt{K_{i,i}} }  + \sum_{j=1}^{2m} \alpha_i K_{i,\#}\\
 &+ 
s \! \times \! A  \left( \sum_{i=1}^m  \left[\Phidis' \left(  \frac{\sum_{j=1}^{2m} \alpha_j K_{i,j}}{ \sqrt{K_{i,i}} }  \right)
\tfrac{K_{i,\#}}{ \sqrt{K_{i,i}} } -  \Phidis' \left(  \frac{\sum_{j=1}^{2m} \alpha_j K_{i+m,j}}{ \sqrt{K_{i+m,i+m}} }  \right) 
  \tfrac{K_{i+m,\#}}{ \sqrt{K_{i+m,i+m}} }\right]\right),
\end{align*}
where
$$
s = \sgn \left( \displaystyle\sum_{i=1}^m  \left[\Phidis \left(  \frac{\sum_{j=1}^{2m} \alpha_j K_{i,j}}{ \sqrt{K_{i,i}} }  \right)    -  \Phidis \left(  \frac{\sum_{j=1}^{2m} \alpha_j K_{i+m,j}}{ \sqrt{K_{i+m,i+m}} }  \right)  \right] \right).
$$


\section{Experiments}
\label{sec:expe}
\subsection{General Setup}
\PBDA\footnote{We made our code available at the following URL: \url{http://graal.ift.ulaval.ca/pbda/}} has been evaluated on a toy problem and a sentiment dataset.
For our experiments, we minimize the objective function using a \emph{Broyden-Fletcher-Goldfarb-Shanno method (BFGS)} implemented in the \emph{scipy} python library\footnote{Available at \url{http://www.scipy.org/}}. 
\noindent\PBDA has been compared with:  
\begin{itemize}
\item \SVM learned only from the source domain, {\it i.e.}, without adaptation. We made use of the SVM-light library \citep{Joachims99}.
\item {\small PBGD3}, presented in Section~\ref{sec:pbgd}, and learned only from the source domain, {\it i.e.}, without adaptation.
\item {\small DASVM} of \citet{BruzzoneM10S}, an iterative domain adaptation algorithm which tries to maximize iteratively a notion of margin on self-labeled target examples. We implemented DASVM with the LibSVM library \citep{libsvm}.
\item {\small CODA} of \citet{ChenWB11}, a co-training domain adaptation algorithm, which looks iteratively for target features related to the training set. We used the implementation provided by the authors.  Note that \citet{ChenWB11} have shown best results on the dataset considered in our Section~\ref{sec:sentiments}.
\end{itemize}
Each parameter is selected with a grid search via a  classical $5$-folds cross-validation (${}^{CV}$) on the source sample for {\small PBGD3} and \SVM, and via a $5$-folds reverse/circular validation (${}^{RCV}$) on the source and the (unlabeled) target samples  for {\small CODA}, {\small DASVM}, and \PBDA. We describe this latter point in the following section.
Note that for \PBDA we search on a $20\times 20$ parameter grid for a $A$ between $0.01$ and $10^6$ and a parameter $C$ between $1.0$ and $10^8$, both on a logarithm scale.

\subsection{A Note about the Reverse Validation}
A crucial question in domain adaptation is the validation of the hyperparameters.
One solution is to follow the principle proposed by \citet{Zhong-ECML10} which relies on the use of a reverse validation approach.
This approach is based on a so-called reverse classifier evaluated on the source domain.
We propose to follow it for tuning the parameters of \PBDA, {\small DASVM} and {\small CODA}.
Note that \citet{BruzzoneM10S} have proposed a similar method, called circular validation, in the context of {\small DASVM}.

Concretely, in our setting, given $k$-folds on the source labeled sample ($S=S_1\cup\ldots\cup S_k$), $k$-folds on the unlabeled target $T$ sample ($T=T_1\cup\ldots\cup T_k$) and a learning algorithm (parametrized by a fixed tuple of hyperparameters),  the reverse cross validation risk on the $i^{\rm th}$ fold is computed as follows.
Firstly, the source set $S\setminus S_i$ is used as a labeled sample and the target set  $T\setminus T_i$ is used as an unlabeled sample for learning a classifier $h'$.
Secondly, using the same algorithm, a reverse classifier $h'^r$ is learned using the \emph{self-labeled} sample $\{(\xbf,  h'(\xbf))\}_{\xbf\in T\setminus T_i}$ as the source set and the unlabeled part of $S\setminus S_i$ as target sample.
Finally, the reverse classifier $h'^r$ is evaluated on $S_i$. 
We summarize this principle on Figure~\ref{fig:rev}. The process is repeated $k$ times to obtain the reverse cross validation risk averaged across all folds.

\begin{figure}[t]
\centering \includegraphics[width=0.7\textwidth]{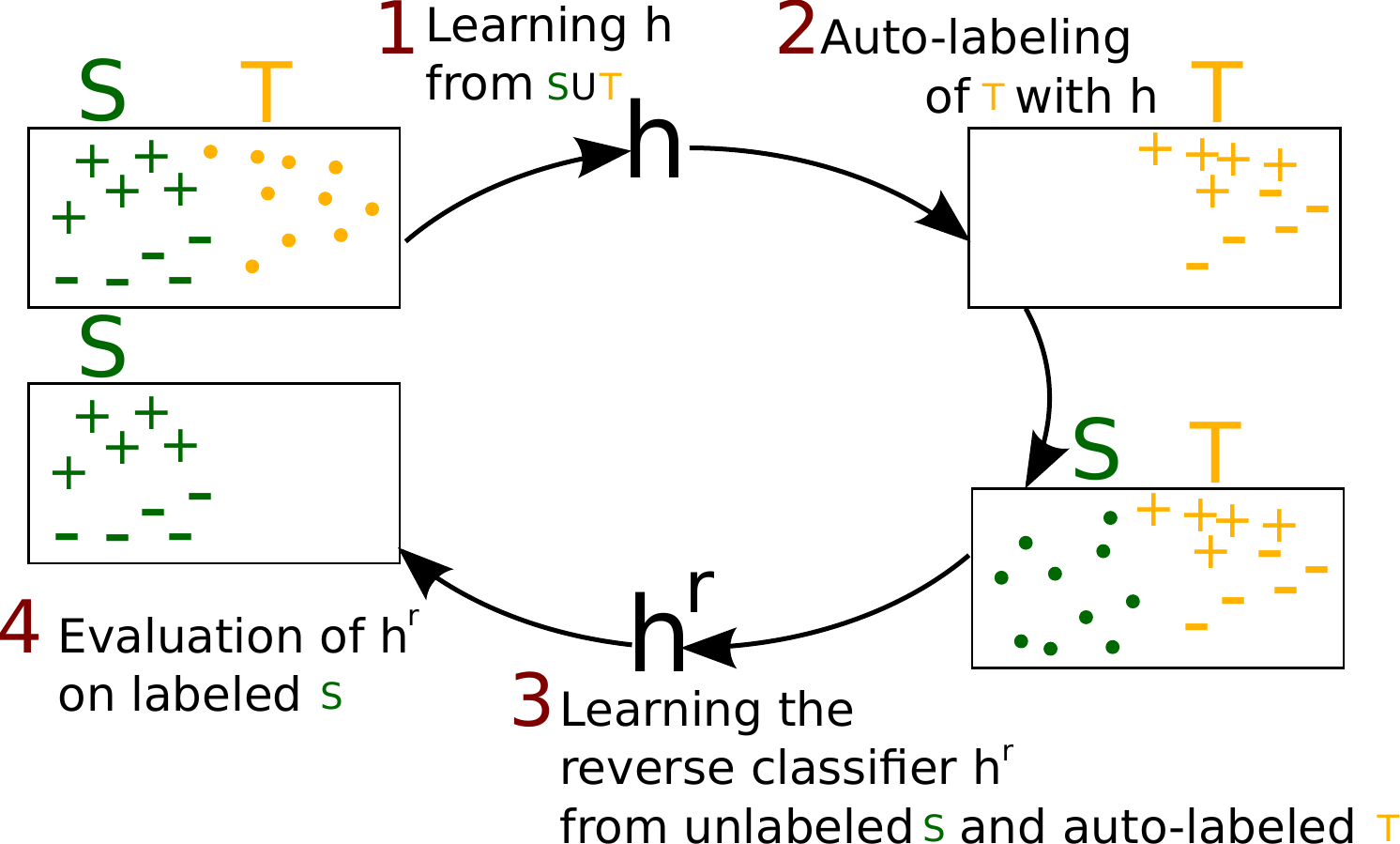}
\caption{\label{fig:rev}The principle of the reverse/circular validation in our setting. 
} 
\end{figure}

\subsection{Toy Problem: Two Inter-Twinning Moons}
The source domain considered here is the classical binary problem with two inter-twinning moons, each class corresponding to one moon (Figure~\ref{fig:moons}).
We then consider seven different target domains by rotating anticlockwise the source domain according to seven angles (from $10\degree$ to $90\degree$).
The higher the angle, the more difficult the problem becomes.
For each domain, we generate $300$ instances ($150$ of each class).
Moreover, to assess the generalization ability of our approach, we evaluate each algorithm on an independent test set of $1, 000$ target points (not provided to the algorithms).
We make use of a Gaussian kernel for all the methods.
Each domain adaptation problem is repeated ten times, and we report the average error rates on Table~\ref{tab:res}.
Note that since {\small CODA} decomposes features for applying co-training, it is not appropriate here (we have only two features).

We remark that our \PBDA provides the best performances except for $50\degree$ and $20\degree$, indicating that \PBDA 
accurately tackles domain adaptation tasks.
It shows a nice adaptation ability, especially for the hardest problem, probably due to the fact that $\des$ is tighter and seems to be a good regularizer in a domain adaptation situation.
 The adaptation versus risk minimization trade-off suggested by Theorem~\ref{theo:pacbayesdabound_bis} appears in Figure~\ref{fig:moons}. Indeed, the plot  illustrates that \PBDA 
accepts to have a lower source accuracy to maintain its performance on the target domain, at least when the source and the target domains are not so different. Note, however, that for large angles, \PBDA prefers to ``focus''  on the source accuracy. We claim that this is a reasonable behavior for a domain adaptation algorithm.

\begin{table}[t]
\caption{Average error rate results for seven rotation angles.\label{tab:res}}
\rowcolors{2}{}{black!10}

\centering
 \begin{tabular}{|c||c|c|c|c|}
          \toprule
    & {\small PBGD3}$^{CV}$      &  {\small SVM}$^{CV}$    &    {\small DASVM}$^{RCV}$    &     {\small PBDA}$^{RCV}$ \\
          \midrule
$\ \  10\degree\ $& ${\it 0}$ &   ${\it 0}$     &     ${\it 0}$     &${\it 0}$  \\
$\ \ 20\degree\ $& $0.088$ & $0.104$     &   ${\it 0}$&  $0.094\,$ \\
$\ \ 30\degree\ $&   $0.210$  &  $0.24$     & $0.259$&  ${\it 0.103}$ \\
$\ \ 40\degree\ $&  $0.273$  &   $0.312$    & $0.284$& ${\it 0.225}$ \\
$\ \ 50\degree\ $&   $0.399$ &   $0.4$   & ${\it 0.334}$ & $0.412$ \\
$\ \ 70\degree\ $&  $0.776$&   $0.764$    &  $0.747$ &   ${\it 0.626}\,$  \\
$\ \ 90\degree\ $&  $0.824$&  $0.828$ &  $0.82$ &     ${\it 0.687}$ \\
\bottomrule
\end{tabular}
\end{table}

\begin{figure}[t]
\centering
\includegraphics[width=0.3\textwidth, trim=17mm 10mm 17mm 14mm, clip]{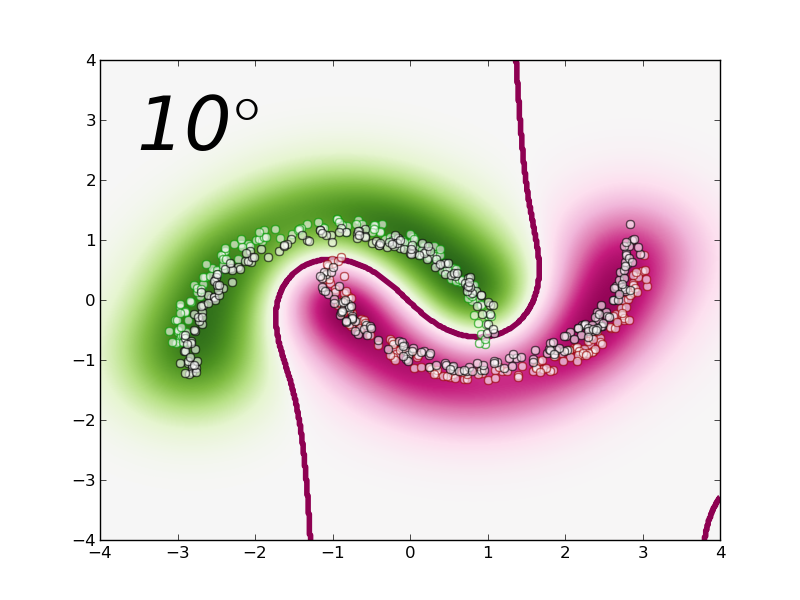}\hfil \includegraphics[width=0.3\textwidth, trim=17mm 10mm 17mm 14mm, clip]{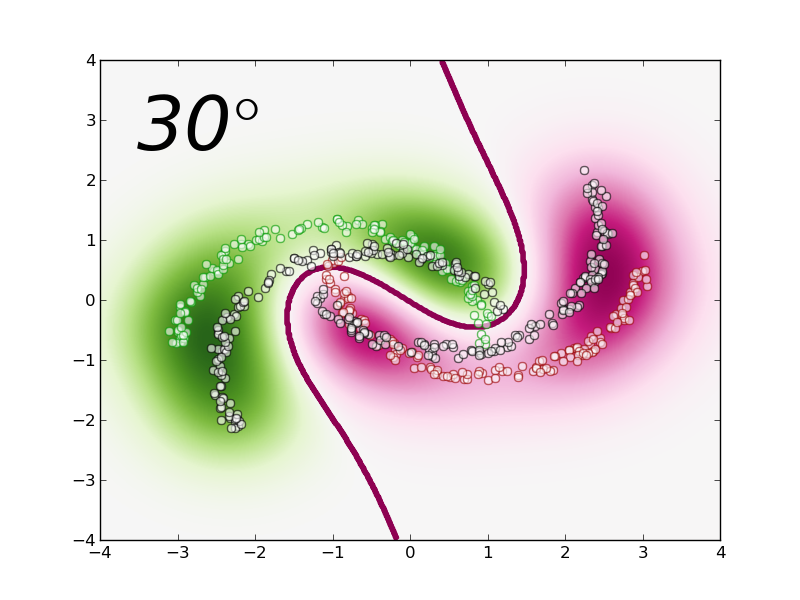}\hfil \includegraphics[width=0.3\textwidth, trim=17mm 10mm 17mm 14mm, clip]{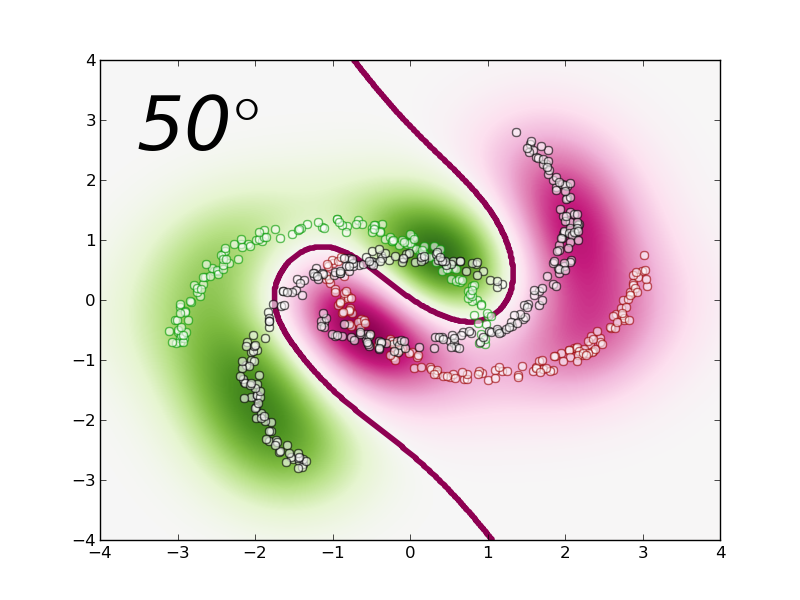}\\[5mm]
\includegraphics[width=0.5\textwidth]{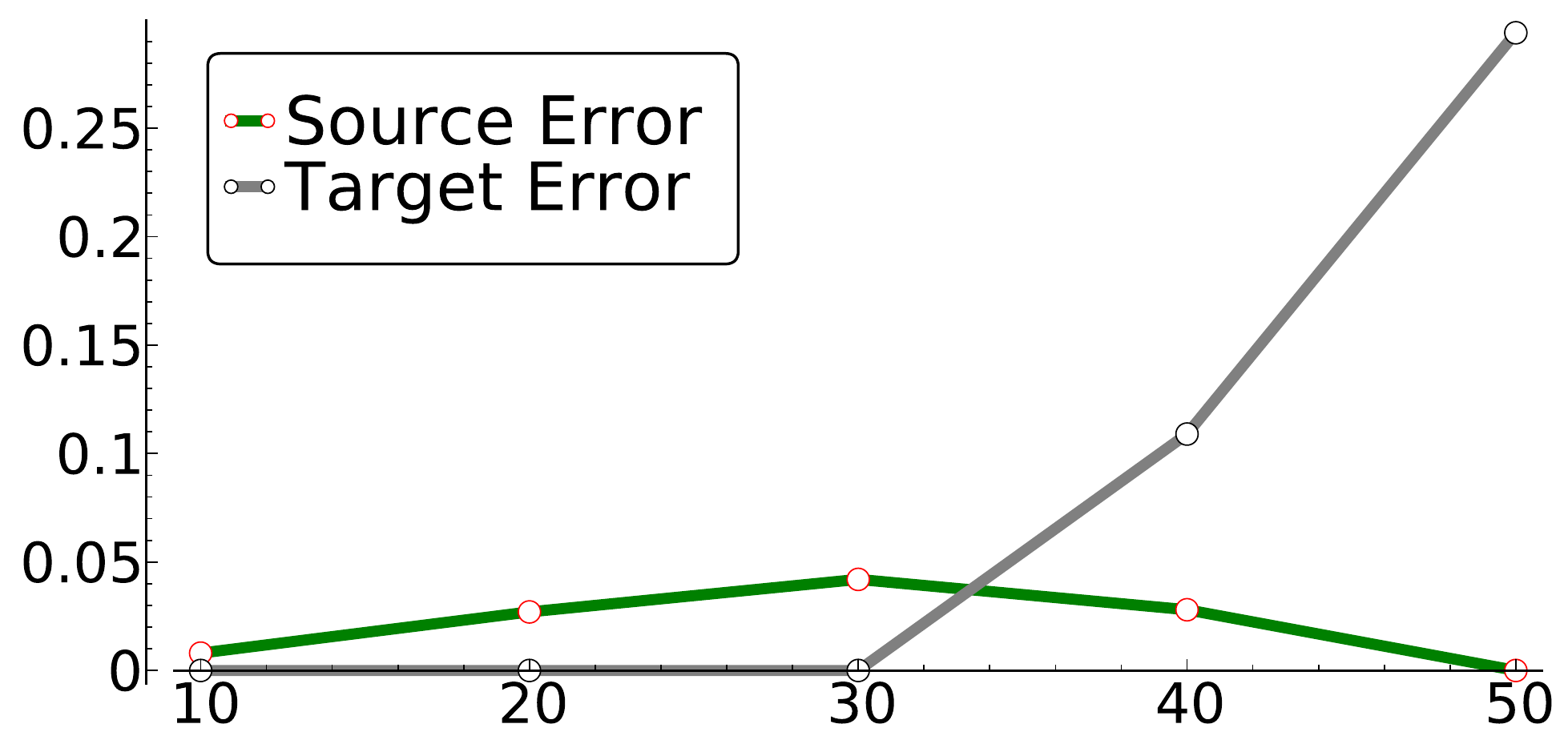}
\caption{Illustration of the decision boundary of \PBDA on three rotations angles for fixed parameters $A = C = 1$.  The two classes of the source sample are green and pink, and target (unlabeled) sample is gray. The bottom plot shows corresponding source and target errors. We intentionally avoid tuning PBDA parameters  to highlight its inherent adaptation behavior.\label{fig:moons}}
\end{figure}

\begin{table}[t]
\centering
\caption{Error rates for sentiment analysis dataset. B, D, E, K respectively denotes  books, DVDs,  electronics,  kitchen. 
\label{tab:res_sentiments}}
\rowcolors{2}{}{black!10}
\begin{tabular}{|c||cc|ccc|}
\toprule
$ $ & {\small PBGD3}$^{CV} $ & {\small SVM}$^{CV} $ & {\small DASVM}$^{RCV} $ & {\small CODA}$^{RCV} $ & {\small PBDA}$^{RCV} $ \\ 
\midrule
$ $B$\rightarrow$D$ $ & $ {\it 0.174}$$ $ & $ $$0.179$$ $ & $ $$0.193$$ $ & $ $${0.181}$$ $ & $ $$0.183$$ $ \\ 
$ $B$\rightarrow$E$ $ & $ $$0.275$$ $ & $ $$0.290$$ $ & $ $${\it 0.226}$$ $ & $ $${0.232}$$ $ & $ $$0.263$$ $ \\ 
$ $B$\rightarrow$K$ $ & $ $$0.236$$ $ & $ $$0.251$$ $ & $ $${\it 0.179}$$ $ & $ $${ 0.215}$$ $ & $ $$0.229$$ $ \\ 
$ $D$\rightarrow$B$ $ & $ $${\it 0.192}$$ $ & $ $$0.203$$ $ & $ $$0.202$$ $ & $ $$0.217$$ $ & $ $$0.197$$ $ \\ 
$ $D$\rightarrow$E$ $ & $ $$0.256$$ $ & $ $$0.269$$ $ & $ $${\it 0.186}$$ $ & $ $${0.214}$$ $ & $ $$0.241$$ $ \\ 
$ $D$\rightarrow$K$ $ & $ $$0.211$$ $ & $ $$0.232$$ $ & $ $$0.183$$ $ & $ $${\it 0.181}$$ $ & $ $$0.186$$ $ \\ 
$ $E$\rightarrow$B$ $ & $ $$0.268$$ $ & $ $$0.287$$ $ & $ $$0.305$$ $ & $ $$0.275$$ $ & $ $${\it 0.232}$$ $ \\ 
$ $E$\rightarrow$D$ $ & $ $$0.245$$ $ & $ $$0.267$$ $ & $ $${\it 0.214}$$ $ & $ $$0.239$$ $ & $ $${0.221}$$ $ \\ 
$ $E$\rightarrow$K$ $ & $ $${\it 0.127}$$ $ & $ $$0.129$$ $ & $ $$0.149$$ $ & $ $${0.134}$$ $ & $ $$0.141$$ $ \\ 
$ $K$\rightarrow$B$ $ & $ $$0.255$$ $ & $ $$0.267$$ $ & $ 0.259$$ $ & $ $${\it 0.247}$$ $ & $ $${\it 0.247}$$ $ \\ 
$ $K$\rightarrow$D$ $ & $ $$0.244$$ $ & $ $$0.253$$ $ & $ $${\it 0.198}$$ $ & $ $$0.238$$ $ & $ $${0.233}$$ $ \\ 
$ $K$\rightarrow$E$ $ & $ $$0.235$$ $ & $ $$0.149$$ $ & $ $$0.157 $ & $ $$0.153$$ $ & $ $${\it 0.129}$$ $ \\ 
\midrule
$ $Average$ $ & $ $$0.226 $ & $ $$0.231 $ & $ $${\it 0.204} $ & $ $$0.210 $ & $ $${0.208} $ \\ 
\bottomrule
\end{tabular}
\end{table}

\subsection{Sentiment Analysis Dataset}
\label{sec:sentiments}
We consider the popular {\it Amazon reviews} dataset \citep{BlitzerMP06}  composed of reviews of four types of {\it Amazon.com}$^{\copyright}$~products (books, DVDs, electronics, kitchen appliances). Originally, the reviews corresponded to a rate between one and five stars and the feature space (of unigrams and bigrams) has on average a dimension of $100, 000$. 
For sake of simplicity and for considering a binary classification task, we propose to follow a  setting similar to  the one proposed by \citet{ChenWB11}.
Then the two possible classes are: $+1$ for the products with a rank higher than $3$ stars, $-1$ for those with a rank lower or equal to $3$ stars.
The dimensionality is reduced in the following way: \citet{ChenWB11} only kept the features that appear at least ten times in a particular DA task (it remains about $40, 000$ features), and pre-processed the data with a standard tf-idf re-weighting.
One type of product is a domain, then we perform twelve domain adaptation tasks. For example, ``books$\rightarrow$DVDs'' corresponds to the task for which books is the source domain and DVDs the target one.
The algorithms use a linear kernel and consider $2, 000$ labeled source examples and $2, 000$ unlabeled target examples.
We evaluate them on separate target test sets proposed by \citet{ChenWB11} (between $3, 000$ and $6, 000$ examples), and we report the results on Table \ref{tab:res_sentiments}.
We make the following observations.

First, as expected, the domain adaptation approaches provide the best average results.
Then, \PBDA is on average better than {\small CODA}, but less accurate than {\small DASVM}. However, \PBDA is competitive: the results are not significantly different from {\small CODA} and {\small DASVM}. 
Moreover, we have observed that \PBDA is significantly faster than {\small CODA} and {\small DASVM}: these two algorithms are based on costly iterative procedures increasing the running time by at least a factor of five in comparison of \PBDA. In fact, the clear advantage of \PBDA is that we jointly optimize the terms of our bound in one step.

\subsection{Combining \PBDA and Representation Learning}

As discussed in the introduction, there exist several families of approaches used to tackle the domain adaptation problem. The present work focuses on the minimization of a distance metric between the source and target distributions. Now, we ask ourselves whether it can be fruitful to combine our \PBDA algorithm with another approach. To do so, we executed \PBDA on top of the Marginalized Stacked Denoising Autoencoders (\msda) introduced by~\citet{Chen12}.

In brief, \msda is an unsupervised algorithm that learns a new representation of the training samples. As a ``denoising autoencoders'' algorithm, it finds a representation from which one can (approximately) reconstruct the original features of an example from its noisy counterpart. The originality of \msda is to learn a representation that allows reconstructing both source and target unlabeled examples. Then, one can execute any supervised learning algorithm on the new representation of source samples, for which the labels are known.

That is, given a source sample $S = \{(\xb^s_i, y^s_i)\}_{i=1}^m$ and a target sample $T = \{(\xb^t_i)\}_{i=1}^{m'}$, \msda takes the unlabeled parts of $S$ and $T$, $\{\xb^s_1, \ldots, \xb^s_m,\xb^t_1, \ldots, \xb^t_{m'}\}$, and learn a feature map $f:X\rightarrow X'$, where $X'$ is a new input space (of real-valued vector).  In \citep{Chen12}, a linear \SVM is executed using $S_f=\{(f(\xb^s_i), y^s_i)\}_{i=1}^m$ as training data, and the hyper-parameter $C$ is selected by standard cross-validation.

We compare the performance of \SVM on \msda representation to \PBDA on the same representations.  That is, we obtain a new representation of both source $S_f = \{(f(\xb^s_i), y^s_i)\}_{i=1}^m$ and target $T_f=\{(f(\xb^t_i))\}_{i=1}^{m'}$ data, using \msda. Then, we execute \PBDA using $S_f$ and $T_f$.

This comparison is done using the {\it Amazon reviews} dataset.
For the sake of comparison, we used the dataset pre-processed by \citet{Chen12}, which is slightly different from the one used in Section~\ref{sec:sentiments}. Indeed, each domain share the same $5,000$ features, and no tf-idf re-weighting is applied.  For each pair source-target, \msda representations are generated using  a \emph{corruption probability} of $50\%$ and a \emph{number of layers} of $5$. Then, \SVM and \PBDA are executed on the same representations. 

\begin{table}[t]
\centering
\caption{Error rates for \msda representations on sentiment analysis dataset.
\label{tab:res_sentiments_msda}}
\rowcolors{2}{}{black!10}
\begin{tabular}{|c||cccc||cc|}
\toprule
 & {\small SVM}$^{CV}$  &  {\small PBDA}$^{CV+RCV}$  & {\small PBDA}$^{RCV}$ & {\small PBDA}$^{CV}$ & {\small SVM}$^{TEST}$ &  {\small PBDA}$^{TEST}$ \\ 
\midrule
B$\rightarrow$D & {\it 0.172} & 0.174 & 0.181 & 0.174 & 0.171 & {\it 0.170} \\ 
B$\rightarrow$E & 0.243 & {\it 0.235} & {\it 0.235} & 0.308 & 0.221 & {\it 0.179} \\ 
B$\rightarrow$K & 0.189 & {\it 0.181} & {\it 0.181} & 0.185 & {\it 0.158} & {\it 0.158} \\ 
D$\rightarrow$B & 0.179 & {\it 0.178} & {\it 0.178} & 0.189 & {\it 0.174} & 0.175 \\ 
D$\rightarrow$E & {\it 0.223} & 0.233 & 0.233 & 0.327 & 0.195 & {\it 0.165} \\ 
D$\rightarrow$K & {\it 0.152} & 0.155 & 0.155 & 0.163 & 0.152 & {\it 0.147} \\ 
E$\rightarrow$B & {\it 0.239} & 0.246 & 0.246 & 0.251 & {\it 0.226} & 0.233 \\ 
E$\rightarrow$D & 0.233 & 0.232 & {\it 0.230} & 0.232 & {\it 0.225} & 0.230 \\ 
E$\rightarrow$K & 0.128 & {\it 0.123} & {\it 0.123} & 0.133 & 0.127 & {\it 0.115} \\ 
K$\rightarrow$B & 0.229 & 0.230 & 0.230 & {\it 0.225} & 0.221 & {\it 0.217} \\ 
K$\rightarrow$D & 0.209 & 0.216 & 0.311 & {\it 0.208} & 0.209 & {\it 0.200} \\ 
K$\rightarrow$E & 0.138 & {\it 0.134} & 0.142 & {\it  0.134} & 0.138 & {\it 0.133} \\
\midrule
Average & {\it 0.195} & {\it 0.195} & 0.204 & 0.211 & 0.185 & {\it 0.177} \\ 
\bottomrule
\end{tabular}
\end{table}

The results are reported in Table~\ref{tab:res_sentiments_msda}. The \PBDA algorithm, when we select the hyperparameter by reverse cross-validation ({\small PBDA}$^{RCV}$), is not always as good as the cross-validated SVM ({\small SVM}$^{CV}$). However, by looking closer at the results, we notice that there often exists hyperparameters for which \PBDA is better on the testing set than the best achievable \SVM (as reported by the columns {\small PBDA}$^{TEST}$ and {\small SVM}$^{TEST}$). 
This suggests that it might be advantageous to mix \msda and \PBDA learning strategies.
However, the hyperparameters selection is still a challenge in domain adaptation, when we do not have any target labels, even if the reverse cross-validation method is a sound strategy. For exploratory purposes, we report on Table~\ref{tab:res_sentiments_msda} the risk of \PBDA while performing the model selection by standard cross-validation ({\small PBDA}$^{CV}$) and while we consider the mean of the cross-validation and the reverse cross-validation score ({\small PBDA}$^{CV+RCV}$). Interestingly, the latter method is a better selection criterion than taking one or the other validation risk separately in this experiment, both being misleading in some situations.\footnote{It is important to point out that experiments on other datasets showed us that the $CV{+}RCV$ method does not systematically outperform the reverse cross-validation method alone.}


\section{Generalization of the PAC-Bayesian Domain Adaptation Theorems to Multisource Domain Adaptation}
\label{sec:multisource}

In this section, we generalize our main analysis to multisource domain adaptation.

\subsection{Multisource Domain Adaptation Setting}

We now consider $n$ different source domains $\{\PSj\}_{j=1}^n$ over $X\times Y$ (along with  $\{\DSj\}_{j=1}^n$ the associated marginal distributions over $X$).
In addition to the target $m'$-sample $T$ with $m'$ unlabeled examples drawn {\it i.i.d.} from the target marginal $\DT$, we have one {\it i.i.d.} source learning sample $S_j$ per domains~$\PSj$ (possibly of different sizes).

Similarly to \citet{BenDavid-MLJ2010}, we study this issue when the relationship between the source domains and the target one is captured by a distribution $v$ over the set of source domains $\{\PSj\}_{j=1}^n$.
This distribution defines a mixture of source domains that we denote by $\multiP$, and its marginal over $X$ by $\multiD$, and $\multiS=\{\Sj\}_{j=1}^n$ corresponds to the set of source samples.
On the source domains, we then consider the following $v$-weighted true error of the Gibbs classifier $\GQ$:
\begin{align*}
\RPv(\GQ)\ &\eqdef\ \esp{\PSj\sim v} \RPSj(\GQ)\\
&=\ \esp{\PSj\sim v} \esp{h\sim \posterior}\RPSj(h)\\
&=\ \sum_{j=1}^n v(\PSj) \esp{h\sim \posterior}\RPSj(h)\,.
\end{align*}
Its empirical counterpart is defined as
\begin{align*}
\RSv(\GQ) 
&\ \eqdef \ \sum_{j=1}^n v(\PSj) \esp{h\sim \posterior}\RSj(h)\,.
\end{align*}

Note that another solution for tackling multisource domain adaptation in a PAC-Bayesian philosophy could be to learn different posterior distribution over $\Hcal$ from different sources.
 Indeed, instead of learning a shared $\posterior$ on every domain (including the target one), we can learn a model for each domain, and then try to learn a good target majority vote over this set of models. In this situation, one could derive a PAC-Bayesian analysis similar to the one provided by \citet{pentina14} for life-long learning.
However, this setting clearly appears to be not pertinent to extend
our one-source domain analysis to multiple sources, since they treat the prior distribution as a random variable, which is not our setting.

\subsection{Generalization of the $\posterior$-Disagreement to Multiple Sources}
One natural solution to generalize the $\posterior$-disagreement of Definition~\ref{def:disagreement} to the multisource setting described in above is to make use of the $v$-weighted sum of each \mbox{$\posterior$-disagreement} between a source distribution and the target one $\espdevant{\DSj\sim v} \des(\DSj,\DT)$, for which we can easily extend Theorem~\ref{thm:pacbayesdabound}.
However, we prefer to consider the following definition that is clearly tighter than the latter one. 
\begin{definition}
Let $\Hcal$ be a hypothesis class. For marginal distributions $\{\DSj\}_{j=1}^n$ and $\DT$ over $X$, any distribution $v$ on $\{\DSj\}_{j=1}^n$, any distribution $\posterior$ on $\Hcal$, the domain disagreement $\des(\multiD,\DT)$ between the mixture of source distribution $\multiD$
and the target distribution $\DT$ is defined by 
\begin{align*}
\des(\multiD,\DT)\ \eqdef\  &\,\Bigg|\, \esp{(h,h')\sim\posterior^2}\bigg[ \RDT(h,h') - 
\esp{\DSj\sim v}\RDSj(h,h') \bigg]\, \Bigg|\\
=\   &\,\Bigg|\,  \RDT(\GQ,\GQ) - 
\esp{\DSj\sim v} \RDSj(\GQ,\GQ)  \,\Bigg| \,.
\end{align*}
\end{definition}
As noticed before, we  trivially have  
\begin{equation} \label{eq:dis_multi_ineq}
 \des(\multiD,\DT)
 \ \leq \ 
 \esp{\DSj\sim v} \des(\DSj,\DT)\,.
 \end{equation}

Therefore, one can use the various PAC-Bayesian bounds presented in Section~\ref{section:PB-dis} to obtain an empirical guarantee over $\des(\multiD,\DT)$ from a collection of observations from each domain.
In particular, Corollary~\ref{cor:bound_dis_rho_multi} below is directly obtained from Theorem~\ref{thm:bound_dis_rho}.

For sake of simplicity, the results presented for the multisource setting suppose that every sample shares the same size $m$. We use the shortcut notation $\multiS\sim(\multiP)^m$ to denote the collection of $n$ source samples of $m$ examples. That is, $\multiS=\{\Sj\}_{j=1}^n$, where $\Sj\sim(\PSj)^m$.

\begin{corollary} \label{cor:bound_dis_rho_multi}
 For any distributions $\{\DSj\}_{j=1}^n$ and $\DT$ over $X$, any set of hypotheses $\Hcal$,  any distribution $v$ over $\{\DSj\}_{j=1}^n$, any prior distribution $\prior$ over $\Hcal$, any $\delta \in (0,1]$, and any real number $\alpha > 0$,  with a probability at least $1 - \delta$ over the choice of 
$\multiS\sim(\multiP)^m$ and $T\sim (D_T)^{m}$, for every $\posterior$ on $\Hcal$, we have
 \begin{align*}
\des(\multiD,\DT)\ \leq\   \frac{2\alpha }{1 -e^{-2\alpha}}  \left[ \esp{\DSj\sim v} \des(\multiS,T)  +  \frac{2\,\KL(\posterior\|\prior)  +  \ln  \frac{2}{\delta} + \ln n}{m\times\alpha} + 1\right] - 1\,.
\end{align*}
\end{corollary}
\begin{proof}
We upper bound the right-hand side of Equation~\eqref{eq:dis_multi_ineq} by upper-bounding each individual term of the expectation using Theorem~\ref{thm:bound_dis_rho}. That is, we bound 
$$v(P_{S_1})\des(S_1,T),\ v(P_{S_2})\des(S_2,T),\ \ldots,\ v(P_{S_n})\des(S_n,T)\,,$$
each one with probability $1-\frac{\delta}{n}$. Thereafter, we regroup these $n$ bounds together to obtain the final result, which stands with probability $1-\delta$.
\end{proof}

The bound given by Corollary~\ref{cor:bound_dis_rho_multi} can suffer from the inequality of  Equation~\eqref{eq:dis_multi_ineq}. 
A better generalization guarantee is given by Theorem~\ref{thm:bound_dis_rho_multi}  below that bounds directly $\des(\multiD,\DT)$, and does not rely on a term ``$\ln n$'' like we have in Corollary~\ref{cor:bound_dis_rho_multi}.
\begin{theorem} \label{thm:bound_dis_rho_multi}
 For any distributions $\{\DSj\}_{j=1}^n$ and $\DT$ over $X$, any set of hypotheses $\Hcal$,  any distribution $v$ over $\{\DSj\}_{j=1}^n$, any prior distribution $\prior$ over $\Hcal$, any $\delta \in (0,1]$, and any real number $\alpha > 0$,  with a probability at least $1 - \delta$ over the choice of $\multiS\sim(\multiP)^m$ and $T\sim (D_T)^{m}$, for every $\posterior$ on $\Hcal$, we have
 \begin{align*}
\des(\multiD,\DT)\, \leq\,   \frac{2\alpha}{1 -e^{-2\alpha}} \left[ \des(\multiS,T)  +  \frac{2\,\KL(\posterior\|\prior)  +  \ln  \frac{2}{\delta}}{m\times\alpha} + 1\right] - 1\,.
\end{align*}
\end{theorem}
\begin{proof}Deferred to Appendix \ref{app:seeger_multi}.
\end{proof}

Note that Theorem~\ref{thm:bound_dis_kl}, Corollary~\ref{cor:bound_dis_kl_pinsker} and Theorem~\ref{thm:bound_dis_rho_allester} can also be rewritten to bound the multisource domain disagreement following the same proof techniques as we used for Theorem~\ref{thm:bound_dis_rho_multi}.

\subsection{Multisource Domain Adaptation Bound for the Stochastic Gibbs Classifier}

Let now generalize the domain adaptation bound of $\RPT(\GQ)$ presented by Theorem~\ref{thm:pacbayesdabound} to our multisource setting.
 
\begin{theorem}
\label{thm:pacbayesdabound_multi}
Let $\Hcal$ be a hypothesis class. We have
$$
\forall \posterior\mbox{ on }\Hcal,\ \forall v\mbox{ on }\{\PSj\}_{j=1}^n, \quad \RPT(\GQ) \ \leq\   
\RPv(\GQ) + \frac{1}{2} \des(\multiD,\DT) + \lambda_\posterior^{v}\,,
$$
where $\lambda_\posterior^{v}$
is the deviation between the expected joint error of $G_\posterior$ on the source domains and the target one:
\begin{eqnarray}\label{eq:lambda_rho_n}
\nonumber\lambda_\posterior^{v}\!\!\!
&\eqdef& \!\!\!
\left|\esp{(h,h')\sim\posterior^2}\!\!\!\left[
\esp{(\xbf,y) \sim \PT}\hspace{-4mm}\zoloss\big( h(\xb), y \big)\zoloss\big( h'(\xb), y \big) -
\esp{\PSj\sim v} \esp{(\xbf,y) \sim \PSj}\hspace{-4mm} \zoloss\big( h(\xb), y \big) \zoloss\big( h'(\xb), y \big)\right]\right|\\
&=& \!\!\!
\Big|\,\ePT(G_\posterior, G_\posterior)-\esp{\PSj\sim v} \ePSj(G_\posterior, G_\posterior)\,\Big|\,.
\end{eqnarray}
\end{theorem}
See Equation~\eqref{eq:eP} for the definition of $ \ePSj(G_\posterior, G_\posterior)$.
\smallskip
\begin{proof}
We follow the same steps as in the proof of Theorem~\ref{thm:pacbayesdabound}. 
Indeed, from Equation~\eqref{eq:rde}, we have
\begin{align*}
\RPT(G_\posterior)\, -\, & \RPv(\GQ)\\ 
 = &\ \ \frac{1}{2} \Big(\RDT(G_\posterior,\GQ)- \esp{\PSj\sim v} \RDSj(G_\posterior,\GQ)\Big) + \Big(\ePT(G_\posterior,G_\posterior)-\esp{\PSj\sim v}  \ePSj(G_\posterior,G_\posterior)\Big) \\
\leq&\ \ 
\frac{1}{2} \Big|\RDT(G_\posterior,\GQ)-  \esp{\PSj\sim v}\RDSj(G_\posterior,\GQ) \Big|  +  \Big|\ePT(G_\posterior, G_\posterior)-\esp{\PSj\sim v} \ePSj(G_\posterior,G_\posterior)\Big|\\
=&\ \
\frac{1}{2} \des(\multiD,\DT) 
+ \lambda_\posterior^{v}\,. 
\end{align*}
\end{proof}

\subsection{PAC-Bayesian Theorem for Multisource Domain Adaptation}
Building on Theorems~\ref{thm:bound_dis_rho_multi} and \ref{thm:pacbayesdabound_multi}, we now present a PAC-Bayesian theorem for multisource domain adaptation.

\begin{theorem}
 \label{theo:pacbayesdabound_catoni_bis_multi}
 For any domains $\{\PSj\}_{j=1}^n$ and $\PT$ (respectively with marginals $\{\DS\}_{j=1}^n$ and $\DT$) over $X\times Y$, any distribution $v$ over $\{\PSj\}_{j=1}^n$,  and for any set $\Hcal$ of hypotheses, for any prior distribution $\prior$ over $\Hcal$, any $\delta\in(0,1]$,  with a probability at least $1-\delta$ over the choice of $\multiS\sim(\multiP)^m$ and $T\sim (D_T)^{m}$, for every $\posterior$ over $\Hcal$, we have
 \begin{align*}
\RPT(G_\posterior)  
\ \leq\  
 c'\,\RSv(\GQ) +  \alpha'\,\tfrac{1}{2} \des(\multiS,T) + 
  \left( \frac{c'}{c} + \frac{\alpha'}{\alpha} \right)  \frac{\KL(\posterior\|\prior)+\ln\frac{3}{\delta}}{m} 
  + \lambda_\posterior^{v} + \tfrac{1}{2}(\alpha'\!-\!1)\,,
 \end{align*}
where $\lambda_\posterior^{v}$ is defined by Equation~\eqref{eq:lambda_rho_n}, 
 and where
$\displaystyle c'\eqdef\frac{c}{1 -e^{-c}}$ \, and \, $\displaystyle \alpha'\eqdef \frac{2\alpha}{1 -e^{-2\alpha}}$\,.
\end{theorem}
\begin{proof}
In Theorem~\ref{thm:pacbayesdabound_multi}, replace $\RSv(G_\posterior)$ and $\des(\multiD,\DT)$ by their upper bound, obtained from Theorem~\ref{thm:pacbayescatoni} and Theorem~\ref{thm:bound_dis_rho_multi}, with $\delta$ chosen respectively as $\frac{\delta}{3}$ and $\frac{2\delta}{3}$.
\end{proof}
Theorem~\ref{theo:pacbayesdabound_catoni_bis_multi} above is a generalization of Theorem~\ref{theo:pacbayesdabound_catoni_bis}. It is straightforward to generalize Theorems~\ref{theo:pacbayesdabound_bis_seeger} and~\ref{theo:pacbayesdabound_bis} as well to the multisource setting.\\

It is important to point out that the above theorem, which naturally generalizes our one-source domain analysis, supposes that the distribution $v$ over $\multiP$ is fixed (or known).
However, we can prove generalization bounds that involve $v$ given a prior distribution $u$ over $\multiP$.
On the one hand, it is possible to derive a result for a distribution $\posterior$ on $\Hcal$ fixed.
On the other hand, such a result can be also derive on $v$ and $\posterior$ at the same time.
These two results can be helpful to derive another kind of approach, and we detail and discuss these bounds in the in Section~\ref{sec:discussion_multisource}.

\subsection{\PBDA for Multisource Domain Adaptation}

Regarding the results of Section~\ref{sec:multisource}, optimizing the PAC-Bayesian multisource domain adaptation bounds of Theorem~\ref{theo:pacbayesdabound_catoni_bis_multi} is equivalent to minimize the following trade-off
$$
C\,m\, 
\RSv(G_{\posterior_\wb}) 
+  A \,m\, \desw(\multiS,T) +  \KL(\posterior_\wb \| \prior_\mathbf{0})\,,
$$
where $$\desw(\multiS,T)  \,=\,  \Big| \,  
\RSv(G_{\posterior_\wb},G_{\posterior_\wb})  -      \RT(G_{\posterior_\wb},G_{\posterior_\wb})\,\Big|,$$
and $\multiS = \{\Sj\}_{j=1}^n = \left\{ \{ (\xbf^s_{ij},y^s_{ij}\}_{i=1}^m \right\}_{j=1}$ are the $n$ source samples coming from the mixture of source domains $\multiP$, and $T= \{(\xb^t_i)\}_{i=1}^m$ is the target sample.
Given the vectors of weights $\vbf=\{v(\PSj)\}_{j=1}^n$ over the source domains, finding the optimal $\posterior_\wb$ is then equivalent to find the vector $\wb$ that minimizes
\begin{equation*}
C \sum_{j=1}^n \sum_{i=1}^m  v(\PSj)\, \Phi\! \!  \left( \!  y^s_{ij} \frac{\wb \cdot \xb^s_{ij}}{\|\xb^s_{ij}\|}  \right)   +  
A  \left| \sum_{i=1}^m  \left[\sum_{j=1}^n  v(\PSj)\, \Phidis\!\!  \left(  \frac{\wb \cdot \xb^s_{ij}}{\|\xb^s_{ij}\|}  \right)    -  \Phidis\! \!  \left(  \frac{\wb \cdot \xb^t_i}{\|\xb^t_i\|}  \right) \right] \right|
  +   \frac{\|\wb\|^2}{2}.
\end{equation*}
Note that if $v$ is a uniform distribution, \emph{i.e.}, every source domain is equally probable, one can solve the above optimization problem using the learning algorithm \PBDA of Section~\ref{sec:pbda}, with $S := \bigcup_{j=1}^n S_j$ as the source sample. In Section~\ref{sec:discussion_multisource}, we discuss the possibility of creating other kinds of learning algorithms, namely by learning $v$, the weights of source distributions.


\section{Discussions}
\label{sec:discussions}
\label{sec:discussion}

In this section, we discuss two points related to this paper.
Firstly, we present two other results in multisource domain adaptation that lead to open-questions related to the deviation of new multisource algorithms.
Secondly, we point out the differences between our new version of the PAC-Bayesian domain adaptation bound (Theorem~\ref{thm:pacbayesdabound}) and the version proposed in \citet{pbda}.

\subsection{Other Results for Multiple Source Domain Adaptation}
\label{sec:discussion_multisource}

In Section~\ref{sec:multisource}, we studied multisource domain adaptation when we suppose that we know the distribution $v$ over $\multiP$.
However, this ideal situation cannot be always verified. Then either one can fix $v$ as the uniform distribution, or one can learn $v$ given a prior distribution $u$ on $\multiP$.
This latter point can be justified by the two following theorems.

Firstly, we can prove a  bound similar to Theorem~\ref{theo:pacbayesdabound_catoni_bis_multi}, but applied on the distribution $v$ on the source domains instead of the distribution $\posterior$ on $\Hcal$.
\begin{theorem}
 \label{theo:pacbayesdabound_catoni_bis_multi_v}
 For any domains $\{\PSj\}_{j=1}^n$ and $\PT$ (respectively with marginals $\{\DS\}_{j=1}^n$ and $\DT$) over $X\times Y$, any prior distribution $u$ over $\{\PSj\}_{j=1}^n$,  and for any set $\Hcal$ of hypotheses, for any fixed distribution\footnote{To avoid confusion with $\posterior$ that we usually want to learn, we denote this fixed distribution $\prior$.} $\prior$ over $\Hcal$, any $\delta\in(0,1]$,  with a probability at least $1-\delta$ over the choice of $\multiS\sim(\multiP)^m$ and $T\sim (D_T)^{m}$, for every $v$ over  $\{\PSj\}_{j=1}^n$, we have
 \begin{align*}
\RPT(G_\prior)  
\ \leq\  c'\,\RSv(G_\prior) +  \alpha' \,\tfrac{1}{2}\des(\multiS,T) + 
  \left( \frac{c'}{c} + \frac{\alpha'}{2\,\alpha} \right)  \frac{\KL(v\|u)+\ln\frac{3}{\delta}}{m}
   + \lambda_\posterior^{v} + \tfrac{1}{2}(\alpha'\!-\!1)
   \,,
 \end{align*}
where $\lambda_\posterior^{v}$ is defined by Equation~\eqref{eq:lambda_rho_n}, 
 and where
$\displaystyle c'\eqdef\frac{c}{1 -e^{-c}}$ \, and \, $\displaystyle \alpha'\eqdef \frac{2\alpha}{1 -e^{-\alpha}}$\,.
\end{theorem}
\begin{proof}Deferred to Appendix~\ref{proof:multi_v}.
\end{proof}

Secondly, it is possible to prove the same kind of generalization bounds for the distribution $v$ over the source domains and the distribution $\posterior$ over $\Hcal$ at the same time.
This result is stated in the next theorem.
 \begin{theorem}
  \label{theo:pacbayesdabound_catoni_bis_multi_vq}
  For any domains $\{\PSj\}_{j=1}^n$ and $\PT$ (respectively with marginals $\{\DS\}_{j=1}^n$ and $\DT$) over $X\times Y$, any prior distribution $u$ over $\{\PSj\}_{j=1}^n$,  and for any set $\Hcal$ of hypotheses, for any prior distribution $\prior$ over $\Hcal$, any $\delta\in(0,1]$,  with a probability at least $1-\delta$ over the choice of $\multiS\sim(\multiP)^m$ and $T\sim (D_T)^{m}$, for every $v$ over  $\{\PSj\}_{j=1}^n$, and every $\posterior$ over $\Hcal$,  we have
   \begin{align*} 
  \RPT(G_\posterior)  
  \ \leq\  c'\,  \RSv(\GQ)  +  \alpha'\, \tfrac{1}{2} \des(\multiS,T) + 
    \left( \frac{c'}{c} + \frac{\alpha'}{\alpha} \right)  \frac{\KL(\posterior\|\prior)+\KL(v\|u)+\ln\frac{3}{\delta}}{m} \\
     {}+ \lambda_\posterior^v + \tfrac{1}{2} (\alpha'\!-\! 1)
     \,,
   \end{align*}
 \end{theorem}
\begin{proof}Deferred to Appendix~\ref{proof:multi_vq_pascal}. 
\end{proof}

These two theorems open the door to the conception of two different algorithms for PAC-Bayesian multisource domain adaptation when we desire to learn both the distributions $v$ on $\multiP$ and $\posterior$ on $\Hcal$.
On the one hand, Theorem~\ref{theo:pacbayesdabound_catoni_bis_multi_v} suggests that one could derive a  two-step algorithm for PAC-Bayesian multisource domain adaptation, according the following principle:
\begin{enumerate}[(i)]
\item Given a fixed distribution $\prior$ over $\Hcal$, we can learn $v$ by minimizing a trade-off between $\RSv(G_\prior)$, $\des(\multiS,T)$ and $\KL(v\|u)$. 
\item Then, for learning $\posterior$, we simply have to optimize \PBDA given this learned $v$.
\end{enumerate}
On the other hand, Theorem~\ref{theo:pacbayesdabound_catoni_bis_multi_vq} implies that we can jointly learn $v$ and $\posterior$ by optimizing the trade-off between $\RSv(\GQ)$, $\des(\multiS,T)$, $\KL(v\|u)$ and $\KL(\posterior\|\prior)$.
This leads to exciting research directions.

\subsection{Comparison with the first PAC-Bayesian domain adaptation bound}
\label{sec:comparaison}

As said in Section~\ref{sec:our_bound}, our PAC-Bayesian domain adaptation bound (of Theorem~\ref{thm:pacbayesdabound}) improves the one provided in \citet{pbda}.
We recall that our bound is expressed as follows.
For every distribution $\posterior$ on $\Hcal$, we have
\begin{align}
\label{eq:new}
\RPT(G_\posterior) \leq   \RPS(G_\posterior) +  \frac{1}{2}\des(\DS,\DT) +  \underbrace{\Big|\, \ePT(G_\posterior, G_\posterior) - \ePS(G_\posterior,G_\posterior) \,\Big|}_{\lambda_{\posterior}}\,.
\end{align}
\citet{pbda} proved the next result.\footnote{The proof of Equation~\eqref{eq:old} relies on several triangle inequalities and on an artificial introduction of a source error term $\RPS(\GQ)$ (see \citet{pbda} for more technical details). Therefore, the proof of Equation~\eqref{eq:new} seems simpler as it is only based the rewriting of the risk  introduced by Equation~\eqref{eq:rde}.}
For every distribution $\posterior$ on $\Hcal$, we have
\begin{align}
\label{eq:old}
\RPT(\GQ) \leq  \RPS(\GQ) +  \des(\DS,\DT) + \underbrace{\RPT(G_{\posterior_T^*}) + \RDT(\GQ,G_{\posterior_T^*}) + \RDS(\GQ,G_{\posterior_T^*})}_{\lambda_{\posterior,\posterior_{\mbox{\tiny $T$}}^*}}\,, 
\end{align}
where $\posterior_{T}^{*} = \argmindevant{\posterior} \RPT(\GQ)$ is the best distribution on the target domain.\\

The improvement of Equation~\eqref{eq:new} over Equation~\eqref{eq:old}  relies on two main points.
On the one hand, our new result contains only the half of $\des(\DS,\DT)$. 
On the other hand, contrary to $\lambda_{\posterior,\posterior_{\mbox{\tiny $T$}}^*}$ of Equation~\eqref{eq:old}, the term $\lambda_{\posterior}$ of Equation~\eqref{eq:new}
does not depend anymore on the best $\posterior_T^*$ on the target domain. This implies that our new bound is not degenerated when the two distributions $\PS$ and $\PT$ are equal (or very close). Conversely, when $\PS = \PT$, the bound of Equation~\eqref{eq:old} gives
\begin{align*}
\RPT(\GQ) \leq  \RPT(\GQ) + \RPT(G_{\posterior_T^*}) + 2 \RDT(\GQ,G_{\posterior_T^*})\,, 
\end{align*}
which is at least $2\RPT(G_{\posterior_T^*})$. Moreover, the term $2 \RDT(\GQ,G_{\posterior_T^*})$ is greater than zero for any  $\posterior$  when the supports of  $\posterior$ and  $\posterior_T^*$ over $\Hcal$ include at least two different classifiers. 


Finally, note that these improvements do not change the form and the philosophy of the PAC-Bayesian theorems of Section~\ref{sec:3PB}, and then of the algorithm \PBDA of Section~\ref{sec:pbda}. Indeed, the only differences stand in $\tfrac{1}{2}\des(\DS,\DT)$ and in the value of $\lambda_\posterior$.


\section{Conclusion and Future Work}
\label{sec:conclu}
In this paper, we define a domain divergence pseudometric that is based on an average disagreement over a set of classifiers, along with consistency bounds for justifying its estimation from samples.
This measure helps us to derive a first PAC-Bayesian bound for domain adaptation.
Moreover, from this bound  we design a well-founded and competitive algorithm (\PBDA) that can jointly optimize the multiple trade-offs implied by the bound for  linear classifiers. 
In addition, we generalize our analysis to multisource domain adaptation, allowing us to take into account information from different source domains according to their relations to the target one.

We think that this PAC-Bayesian analysis opens the door to develop new domain adaptation methods  by making use of the possibilities offered by the PAC-Bayesian theory, and gives rise to  new interesting directions of research, among which the following ones.

Firstly, the PAC-Bayesian approach allows one to deal with an {\it a priori} belief on what are the best classifiers; in this paper we opted for a non-informative prior that is a Gaussian centered at the origin of the linear classifier space. The question of finding a relevant prior in a domain adaptation situation is an exciting direction which could also be exploited when some few target labels are available.
Moreover, as pointed out by \citet{pentina14}, this notion of prior distribution could modelize information learned from previous tasks. This suggests that we can extend our multisource analysis to issues related to lifelong learning where the objective is
to perform well on future tasks, for which so far no data has been observed \citep{ThrunM95}.

Another promising issue is to address the problem of the hyperparameter selection. Indeed, the adaptation capability of our algorithm \PBDA could be even put further with a specific PAC-Bayesian validation procedure. An idea would be to propose a kind of (reverse) validation technique that takes into account some particular prior distributions. Another possible solution could be to explicitly control the neglected term in the domain adaptation bound.
This is also linked with model selection for domain adaptation tasks.

Besides, deriving a result similar to Equation~\eqref{eq:C-bound} (the $C$-bound) for domain adaptation could be of high interest. 
Indeed, such an approach considers the first two moments of the margin of the weighted majority vote. This could help us to take into account both  a kind of margin information over unlabeled data and the distribution disagreement (these two elements seem of crucial importance in domain adaptation).

\subsubsection*{Acknowledgments}
This work was supported in part by the French projects VideoSense {\scriptsize ANR-09-CORD-026} and  LAMPADA {\scriptsize ANR-09-EMER-007-02},  and in part by NSERC discovery grant {\scriptsize 262067}, and by the European Research Concil under the European Unions
Seventh Framework Programme (FP7/2007-2013)/ERC grant agreement no 308036.
 Computations were performed on Compute Canada and Calcul Qu\'ebec infrastructures (founded by CFI, NSERC and FRQ). We thank Christoph Lampert and Anastasia Pentina for helpful discussions.\\
 A part of the work of this paper was carried out while E. Morvant was affiliated with IST Austria.
This work was carried out while P. Germain was affiliated with D\'epartement d'informatique et de g\'enie logiciel,
Universit\'e Laval, Qu\'ebec, Canada.

\appendix

\appendix
\section{Some Tools} \label{section:tools}


\begin{lemma}[Markov's inequality]
\label{theo:markov}
Let $Z$ be a random variable and $t\geq 0$,  then \\[-2mm]
$$
P{(|Z|\geq t)} \ \leq \ \esp{}(|Z|)\,/\,t\,.
$$
\end{lemma}

\begin{lemma}[Jensen's inequality]
\label{theo:jensen}
Let $Z$ be an integra\-ble real-valued random variable and $g(\cdot)$ any function.\\

\noindent
If $g(\cdot)$ is convex, then
$$
\quad g(\esp{}[Z])\ \leq\ \esp{}[g(Z)]\,.
$$
If $g(\cdot)$ is concave, then
$$
\quad g(\esp{}[Z])\ \geq\ \esp{}[g(Z)]\,.
$$
\end{lemma}

\begin{lemma}[\citet{Maurer04}]
\label{lem:maurer2}
Let $X = ( X_1,\dots,X_m)$ be a vector of {\it i.i.d.} random variables, $0\leq X_i\leq 1$, with $\esp{}X_i = \mu$.  Denote $X' = ( X'_1,\dots,X'_m)$, where~$X_i'$ is the unique Bernoulli ($\{0,1\}$-valued) random variable with $\esp{} X_i' = \mu$. If $f:[0,1]^n\rightarrow \mathbb{R}$ is convex, then\\[-2mm]
$$
\esp{} [f(X)] \ \leq \ \esp{} [f(X')]\,.
$$
\end{lemma}

\begin{lemma}[from Inequalities (1) and (2) of \citet{Maurer04}]
\label{lem:maurer}
Let $m \geq  8$, and  $X = ( X_1,\dots,X_m)$ be a vector of {\it i.i.d.} random variables, $0\leq X_i\leq 1$. 
 Then 
$$
\sqrt{m}\ \leq\ \esp{} \exp \left[m\, \kl\left(\frac{1}{m}\sum_{i=1}^n X_i\,\Big\|\,\esp{}[X_i]\right)\right]\ \leq\ 2\sqrt{m}\,.
$$ 
\end{lemma}

\begin{lemma}[Change of measure inequality] \label{lem:change-measure}
For any set $\Hcal$, for any distributions $\prior$ and $\posterior$ on $\Hcal$, and for any measurable function $\phi:\Hcal \to \mathbb{R}$,  we have
\begin{equation*}
\esp{f\sim \posterior} \phi(f) \ \leq \ \KL(\posterior\|\prior) + \ln \left( \esp{f\sim \prior} e^{\phi(f)} \right) \,.
\end{equation*}
\end{lemma}

\begin{lemma} \label{lem:2KL}
Given any set $\Hcal$, and any distributions $\prior$ and $\posterior$
on $\Hcal$, let  $\hat{\posterior}$ and $\hat{\prior}$ two distributions over $\Hcal^2$ such that $\hat{\posterior}(h,h') \eqdef \posterior(h)\posterior(h')$ and $\hat{\prior}(h,h') \eqdef \prior(h)\prior(h')$. Then
\begin{equation*}
\KL(\hat{\posterior}\|\hat{\prior})  \ =\   2\,\KL(\posterior\|\prior)\,.
\end{equation*}
\end{lemma}
\begin{proof}
\begin{align*}
\nonumber \KL(\hat{\posterior}\|\hat{\prior}) \ &=\ \esp{(h,h')\sim \posterior^2 } \ln \frac{\posterior(h)\posterior(h')}{\prior(h)\prior(h')} \\
\nonumber&=\  \esp{h\sim \posterior } \ln \frac{\posterior(h)}{\prior(h)} + \esp{h'\sim \posterior } \ln \frac{\posterior(h')}{\prior(h')} \\
&=\  2\esp{h\sim \posterior } \ln \frac{\posterior(h)}{\prior(h)}\\
& =\ 2\,\KL(\posterior\|\prior)\,.\\[-1cm]
\end{align*} 
\end{proof}

\newcommand{\DST}{D_{{ S \times  T}}}
\renewcommand{\ST}{{S \times  T}}
\newcommand{\RDSTa}{{R_{{\DST}}^{(1)}}}
\newcommand{\RSTa}{{R_{{S \times T}}^{(1)}}}
\newcommand{\Xh}{{X_{\hat{h}}}}

\section{Proof of Theorem \ref{thm:bound_dis_kl}}
\label{app:seeger}

\begin{proof}
 Firstly, we propose to upper-bound
\begin{align*}
d^{(1)} \, \eqdef  \esp{(h,h')\sim \posterior^2 }   \Big[ \RDS (h,h')  -  \RDT (h,h') \Big]
\end{align*}
by its empirical counterpart
\begin{align*}
d^{(1)}_{\mbox{\tiny $S  \times   T$}} \, \eqdef  \esp{(h,h')\sim \posterior^2 }  \Big[ \RS(h,h')  -  \RT(h,h') \Big]\,.
\end{align*}
To achieve this, we consider an ``abstract'' classifier $\hat{h}  \eqdef (h,h')\in\Hcal^2$ chosen according a distribution $\hat{\posterior}$, with \mbox{$\hat{\posterior}(\hat{h})  =  \posterior(h) \posterior(h')$}.
Let us define the ``abstract'' loss of $\hat{h}$ on a pair of examples $(\xbs,\xbt)\sim \DST = D_S\times D_T$ by
$$\loss_{d^{(1)}} (\hat{h},\xbs  ,\xbt)  \ \eqdef\  \frac{1 + \zoloss (h(\xbs ),  h'(\xbs ) ) -  \zoloss (h(\xbt),  h'(\xbt) )}{2}\,.$$
Therefore, the ``abstract'' risk of $\hat{h}$ on the joint distribution is defined as
$$\RDSTa(\hat{h})\ =\, \esp{\xbs\sim D_S} \esp{\xbt\sim D_T}      \loss_{d^{(1)}}(\hat{h}, \xbs ,\xbt)\,,$$
which empirical counterpart is
$$\RSTa(\hat{h})\ =\,\esp{(\xbs ,\xbt)\sim \ST}      \loss_{d^{(1)}} (\hat{h}, \xbs ,\xbt)\,. $$
The error of the related Gibbs classifier of these two quantities are
\begin{eqnarray} \label{eq:abstract_gibbs_errors}
 \RDSTa (G_{\hat{\rho}})\, = \esp{\hat{h}\sim \hat{\posterior} } \RDSTa(\hat{h})
\quad\mbox{  and }\quad
\RSTa (G_{\hat{\rho}})\, = \esp{\hat{h}\sim \hat{\posterior} } \RSTa(\hat{h})\,.
 \end{eqnarray}
 It is easy to show that
 \begin{eqnarray}
 \label{eq:d_et_Gibbs}
d^{(1)}=2\RDSTa (G_{\hat{\rho}})-1 
\quad\mbox{  and }\quad
d^{(1)}_{\mbox{\tiny $S  \times   T$}}=2\RSTa (G_{\hat{\rho}})-1\,.
\end{eqnarray}

\medskip
\noindent
Now, let us consider the non-negative random variable
 $\esp{\hat{h}\sim \hat{\prior} } \expo{ m \kl\left( \RSTa(\hat{h}) \big\| \RDSTa(\hat{h})  \right) }.$
 
\smallskip
We apply Markov's inequality (Lemma~\ref{theo:markov}). For every $\delta\in(0,1]$, with a probability at least $1-\delta$ over the choice of $\ST \sim(\DST)^m$, we have
\begin{align*}
\esp{\hat{h}\sim \hat{\prior} } \expo{ m \kl\left( \RSTa(\hat{h}) \big\| \RDSTa(\hat{h})  \right)  } 
\ &\leq \ 
\frac{1}{\delta}\,  \esp{\ST\sim (\DST)^m} \esp{\hat{h}\sim \hat{\prior} } \expo{  m \kl\left( \RSTa(\hat{h}) \big\| \RDSTa(\hat{h})  \right)}\\
&=\ 
\frac{1}{\delta}\,  \esp{\hat{h}\sim \hat{\prior} }  \esp{\ST\sim (\DST)^m} \expo{  m \kl\left( \RSTa(\hat{h}) \big\| \RDSTa(\hat{h})  \right)}\\
&\leq\  
\frac{1}{\delta}\,  \esp{\hat{h}\sim \hat{\prior} } 2\sqrt{m}\,,
\end{align*}
where the last inequality comes from the Maurer's lemma (Lemma~\ref{lem:maurer}).\\
By taking the logarithm of each outermost sides of the previous equation,
we then obtain
\begin{align*} 
\nonumber \ln\left[ \esp{\hat{h}\sim \hat{\prior} } 
 \expo{ m \kl\left( \RSTa(\hat{h}) \big\| \RDSTa(\hat{h})  \right)  } \right]\ 
\leq \ \ln \frac{2\sqrt{m}}{\delta}\,.
\end{align*}

Let us now find a lower bound of the left side of the last equation by using 
the change of measure inequality (Lemma~\ref{lem:change-measure}) and the Jensen inequality (Lemma~\ref{theo:jensen}) on the convex function~$\kl(\cdot\|\cdot)$. We have
\begin{align*}
\ln\left[ \esp{\hat{h}\sim \hat{\prior} } 
 \expo{ m \kl\left( \RSTa(\hat{h}) \big\| \RDSTa(\hat{h})  \right)  } \right]\ 
&\geq\ \esp{\hat{h}\sim \hat{\posterior} }  m\, \kl\left( \RSTa(\hat{h}) \,\big\|\, \RDSTa(\hat{h})  \right)   -\KL(\hat{\posterior}\|\hat{\prior}) \\
&\geq\ m\, \kl\left( \esp{\hat{h}\sim \hat{\posterior} }\RSTa(\hat{h})\, \big\|\, \esp{\hat{h}\sim \hat{\posterior} }\RDSTa(\hat{h})  \right)   -\KL(\hat{\posterior}\|\hat{\prior})\\
&=\ m\, \kl\left( \RSTa(G_{\hat{\posterior}})\, \big\|\,\RDSTa(G_{\hat{\posterior}})  \right)   -2\,\KL(\posterior\|\prior)\,. 
 \end{align*} 
Note that the last equality is obtained from Equation~\eqref{eq:abstract_gibbs_errors} and Lemma~\ref{lem:2KL}.

\medskip
\noindent
We finally obtain
\begin{equation*}
\kl \left(  \RSTa (G_{\hat{\posterior}}) \big\|\RDSTa(G_{\hat{\posterior}})   \right) 
\ \leq\   \frac{1}{m} \left[2\,\KL(\posterior\,\|\,\prior) +  \ln  \frac{2\sqrt{m}}{\delta}\right]  .
\end{equation*}
With  Equation~\eqref{eq:d_et_Gibbs}, the previous line gives us a bound on $d^{(1)}$ from its empirical counterpart~$d^{(1)}_{\mbox{\tiny $S  \times   T$}}$. 
Hence, with probability at least $1 - \delta$ over the choice of $S\times T\sim(\DS\times \DT)^m$, 
\begin{equation*} 
\kl \left(  \tfrac{d^{(1)}_{\mbox{\tiny $S  \times   T$}} +1}{2} \Big\|\tfrac{d^{(1)}+1}{2}  \right) 
\ \leq\  \frac{1}{m} \left[2\,\KL(\posterior\,\|\,\prior) +  \ln  \frac{2\sqrt{m}}{\delta}\right]  .
\end{equation*}

\noindent
Lemma~\ref{lem:claim} (stated below) gives
\begin{equation*} 
\kl \left(   \tfrac{|d^{(1)}_{\mbox{\tiny $S  \times   T$}}\hspace{-0.3mm}| +1}{2} \Big\|\tfrac{|d^{(1)}\hspace{-0.3mm}|+1}{2}   \right) 
\ \leq  \  \frac{1}{m} \left[2\,\KL(\posterior\,\|\,\prior) +  \ln  \frac{2\sqrt{m}}{\delta}\right]  ,
\end{equation*}
which, since 
$
|d^{(1)}|  = \des(\DS,\DT) \  \ \mbox{and}\ \ |d^{(1)}_{\mbox{\tiny $S  \times   T$}}|  = \des(S,T)\,,
$
implies the result.
\end{proof}

\begin{lemma} \label{lem:claim}
 For $a,b\in[-1,+1]$, we have
\begin{equation*}
\kl \left(   \tfrac{1+|a|}{2}\, \Big\|\, \tfrac{1+|b|}{2}  \right) 
\ \leq \ 
\kl \left(   \tfrac{1+a}{2}\, \Big\|\, \tfrac{1+b}{2}  \right) \,.
\end{equation*}
\end{lemma}
\begin{proof}
There are four cases to consider.

\begin{description}
\item[\textbf{Case \MVOne:}]  Let $a \geq 0$ and $b \geq 0$.\\
This first case is trivial, since $|a|=a$ and $|b|=b$.\\

\item[\textbf{Case \MVTwo:}] Let $a \leq 0$ and $b \leq 0$. 
\\ 
This case reduces to \textbf{Case \MVOne} because 
$\kl(q\|p) = \kl(1\!-\!q\|1\!-\!p)$ for all
$(q,p)  \in [0,1]^2$ .\\
 Then
\begin{equation*} 
\kl \left(   \tfrac{1+|a|}{2}\, \Big\|\, \tfrac{1+|b|}{2}  \right)
\ = \  
\kl \left(   \tfrac{1-a}{2}\, \Big\|\, \tfrac{1-b}{2}  \right) 
\ = \  
\kl \left(   \tfrac{1+a}{2}\, \Big\|\, \tfrac{1+b}{2}  \right)
  \,.
\end{equation*}

\item[\textbf{Case \MVThree:}]  Let $a \leq 0$ and $b \geq 0$.\\
 From straightforward calculations, we show that
\begin{eqnarray*}
\nonumber
\kl \left(  \tfrac{1+|a|}{2} \Big\|\tfrac{1+|b|}{2}   \right)
 - 
\kl \left(  \tfrac{a +1}{2} \Big\|\tfrac{1+b}{2}   \right)  
 \hspace{-35mm}& &
\\[2mm]
&=&
\kl \left(  \tfrac{1-a}{2} \Big\|\tfrac{1+b}{2}   \right)
 - 
\kl \left(  \tfrac{1+a}{2} \Big\|\tfrac{1+b}{2}   \right)\\[2mm]
\nonumber
&=&
\left( \tfrac{1-a}{2} - \tfrac{1+a}{2}\right) \ln\left(\frac{1}{\tfrac{1+b}{2}}\right)  
+  \left(   \Big(1 - \tfrac{1-a}{2}\Big) - \Big(1 - \tfrac{1+a}{2}\Big) \right) \ln\left(\frac{1}{1-\tfrac{1+b}{2}}\right)\\[2mm]
\nonumber
&=&
 -a  \ln\left(\frac{1}{\tfrac{1+b}{2}}\right) \ + \
 a  \ln\left(\frac{1}{1-\tfrac{1+b}{2}}\right)
 \ = \ 
 -a  \ln\left(\frac{1}{\tfrac{1+b}{2}}\right) \ + \
a\ln\left(\frac{1}{\tfrac{1-b}{2}}\right)\\[2mm]
\nonumber
&=&
a\ \ln\left(\tfrac{1+b}{1-b}\right)\\[2mm]
\nonumber 
&\leq&
0\,.
\end{eqnarray*}

\item[\textbf{Case \MVFour:}]  Let $a \geq 0$ and $b \leq 0$.\\
This case reduces to \textbf{Case \MVThree}, since
$\kl(q\|p) = \kl(1\!-\!q\|1\!-\!p)$ for all
$(q,p)  \in [0,1]^2$ .\\
 Hence,
 \begin{equation*} 
 \kl \left(   \tfrac{1+|a|}{2}\, \Big\|\, \tfrac{1+|b|}{2}  \right)
 \ = \  
 \kl \left(   \tfrac{1+a}{2}\, \Big\|\, \tfrac{1-b}{2}  \right) 
 \ \leq \  
 \kl \left(   \tfrac{1+a}{2}\, \Big\|\, \tfrac{1+b}{2}  \right)
   \,.
 \end{equation*}
\end{description}
\end{proof}
\pagebreak
\section{Detailed Proof of Theorem \ref{thm:bound_dis_rho}}
\label{app:catoni}

\begin{proof}
Similarly as in the proof of Theorem~\ref{thm:bound_dis_kl} (see Appendix~\ref{app:seeger}), we will first  bound
\begin{align*}
d^{(1)} \, \eqdef  \esp{(h,h')\sim \posterior^2 }   \Big[ \RDS (h,h')  -  \RDT (h,h') \Big]
\end{align*}
by its empirical counterpart.\\
Refer to the proof of Theorem~\ref{thm:bound_dis_kl} for the definitions of $\RDSTa(\hat{h})$ and $R_{{\DST}}^{(1)} (G_{\hat{\rho}})$, as well as their empirical counterparts $\RSTa(\hat{h})$  and  $R_{{\ST}}^{(1)} (G_{\hat{\rho}})$.

As $\loss_{d^{(1)}}$ lies in $[0,1]$, we can bound $R_{{\DST}}^{(1)} (G_{\hat{\rho}})$ following the proof process of Theorem~\ref{thm:pacbayescatoni} (with $c = 2\alpha$). To do so, we define the convex function,
\begin{equation} \label{eq:Fcal}
\Fcal(x) \ \eqdef \ - \ln\left[\,1-(1-e^{-2\alpha})\, x\,\right],
\end{equation}
and consider the non-negative random variable 
$\esp{\hat{h}\sim \hat{\prior} } \expo{ m \left(\Fcal(\RDSTa(\hat{h}))-2\alpha\RSTa(\hat{h})  \right) }.$

We apply Markov's inequality (Lemma~\ref{theo:markov}). For every $\delta\in(0,1]$, with a probability at least $1 - \frac{\delta}{2}$ over the choice of $\ST \sim(\DST)^m$, we have
\begin{align*}
\esp{\hat{h}\sim \hat{\prior} } \expo{ m \left(\Fcal(\RDSTa(\hat{h}))-2\alpha\RSTa(\hat{h})  \right)} 
 &\leq \ \frac{2}{\delta}\, \esp{\ST\sim (\DST)^m} \esp{\hat{h}\sim \hat{\prior} } \expo{  m \left(\Fcal(\RDSTa(\hat{h}))-2\alpha\RSTa(\hat{h})  \right)  }\\
&=\ \frac{2}{\delta}\, \esp{\hat{h}\sim \hat{\prior} } 
  \expo{ m\Fcal(\RDSTa(\hat{h}))}\esp{\ST\sim (\DST)^m}   \expo{- 2m\alpha\RSTa(\hat{h}) } \,.
\end{align*}
By taking the logarithm on each side of the previous inequality, we obtain 
\begin{align} 
\ln\left[ \esp{\hat{h}\sim \hat{\prior} } \hspace{-1.5mm} \expo{ m \left(\Fcal(\RDSTa(\hat{h}))-2\alpha\RSTa(\hat{h})  \right)  } \right] 
= \ln\left[\frac{2}{\delta} \esp{\hat{h}\sim \hat{\prior} }  \hspace{-1.5mm}
 \expo{ m\Fcal(\RDSTa(\hat{h}))}  \hspace{-3mm} \esp{\ST\sim (\DST)^m}  \hspace{-7mm} \expo{- 2m\alpha\RSTa(\hat{h}) } \right] .
 \label{eq:kiwi}
\end{align}

For a classifier $\hat{h}$, let us define a random variable $\Xh$ that follows a binomial distribution of $m$ trials with a probability of success $\RDSTa(\hat{h})$ denoted by $B\big(m, \RDSTa(\hat{h}) \big)$.
Lemma~\ref{lem:maurer2} gives
\begin{align*} 
 \hspace{-6.5mm} \esp{\ST\sim (\DST)^m} \expo{- 2m\alpha\RSTa(\hat{h}) } 
\ & \leq\ \esp{\Xh\sim B(m, \RDSTa(\hat{h}))} \expo{- 2\alpha \Xh } \\
  & =\  \sum_{k=0}^m\ \prob{\Xh\sim B(m, \RDSTa(\hat{h}))}    \Big(\Xh=k \Big)  e^{- 2\alpha k}\\
& =\ \sum_{k=0}^m {\textstyle{m \choose k}}  \big( \RSTa(\hat{h})\big)^k  \big(1-\RSTa(\hat{h})\big)^{m-k}  e^{- 2\alpha k} \\
& =\ \sum_{k=0}^m {\textstyle{m \choose k}}  \big( \RSTa(\hat{h})e^{- 2\alpha }\big)^k  \left(1-\RSTa(\hat{h})\right)^{m-k}  \\
 & =\  \left[ \RSTa(\hat{h})e^{- 2\alpha } + \left(1-\RSTa(\hat{h}) \right) \right]^{m} .
\end{align*}
The last line result, together with the choice of $\Fcal$ (Equation~\eqref{eq:Fcal}), leads to
\begin{align*}
\esp{\hat{h}\sim \hat{\prior} }  
 \expo{ m\Fcal(\RDSTa(\hat{h}))} & \esp{\ST\sim (\DST)^m} \expo{- 2m\alpha\RSTa(\hat{h}) } \\
 %
 \leq\  &\ \esp{\hat{h}\sim \hat{\prior} } 
    e^{ m\Fcal(\RDSTa(\hat{h}))} \left[ \RSTa(\hat{h})e^{- 2\alpha } + \left(1-\RSTa(\hat{h}) \right) \right]^{m}  \\
  =\ &\ \esp{\hat{h}\sim \hat{\prior} }  1 \, = \, 1\,.   
\end{align*}
We can now upper bound Equation~\eqref{eq:kiwi} simply by
\begin{align*} 
\ln\left[ \esp{\hat{h}\sim \hat{\prior} } \expo{ m \left(\Fcal(\RDSTa(\hat{h}))-2\alpha\RSTa(\hat{h})  \right)  } \right] \ 
 \leq \ \ln\frac{2}{\delta}\,.
 \end{align*} 
Let us now find a lower bound of the left side of the last equation by using 
the change of measure inequality (Lemma~\ref{lem:change-measure}) and the Jensen's inequality (Lemma~\ref{theo:jensen}) on the convex function~$\Fcal$:
\begin{align*}
\ln\left[\, \esp{\hat{h}\sim \hat{\prior} } 
 \expo{ m \kl\left( \RSTa(\hat{h}) \big\| \RDSTa(\hat{h})  \right)  } \right]\ 
&\geq\ 
 \esp{\hat{h}\sim \hat{\posterior} } m \left(\Fcal(\RDSTa(\hat{h}))-2\alpha\RSTa(\hat{h})  \right) -\KL(\hat{\posterior}\|\hat{\prior})\\
&\geq\ 
 m \left[\Fcal\left( \esp{\hat{h}\sim \hat{\posterior} }\RDSTa(\hat{h})\right)-2\alpha \esp{\hat{h}\sim \hat{\posterior} } \RSTa(\hat{h})  \right] -\KL(\hat{\posterior}\|\hat{\prior})\\
&= \ 
 m \Fcal\left( \RDSTa(G_{\hat{\posterior}}))-2m\alpha\RSTa(G_{\hat{\posterior}})   -2\,\KL(\posterior\|\prior)\right)\,.
 \end{align*} 
The last equality is obtained from Equation~\eqref{eq:abstract_gibbs_errors} and Lemma~\ref{lem:2KL}.
This, in turn, implies 
\begin{align*} 
&\Fcal( \RDSTa(G_{\hat{\posterior}}))
 \leq 
 2\alpha\RSTa(G_{\hat{\posterior}}) + \frac{2\,\KL(\posterior\|\prior)  + \ln\frac{2}{\delta}}{m}  \,.
  \end{align*}
Now, by isolating $\RDSTa (G_{\hat{\posterior}})$, we obtain
\begin{align*} 
\RDSTa (G_{\hat{\posterior}})  \leq   \frac{1}{1 - e^{-2\alpha}}    \left[ 1 - e^{  -\left( 2\alpha\RSTa(G_{\hat{\posterior}})\, +\, \frac{1}{m} \left( 2\,\KL(\posterior\|\prior)  + \ln\frac{2}{\delta}\right)  \right)}   \right] ,
 \end{align*} 
and, from the inequality $1-e^{-x}\leq x$\,,
\begin{align*} 
\RDSTa(G_{\hat{\posterior}})  \leq   \frac{1}{1 - e^{-2\alpha}}   \left[ 2\alpha\RSTa(G_{\hat{\posterior}}) +  \frac{ 2\,\KL(\posterior\|\prior)  + \ln\frac{2}{\delta} }{m} \right] .
 \end{align*} 
It then follows from Equation~\eqref{eq:d_et_Gibbs} that, with probability at least $1 - \tfrac{\delta}{2}$ over the choice of $S\times T\sim(\DS\times \DT)^m$, we have
\begin{equation} \label{eq:catoni_primo}
\frac{d^{(1)}+1}{2} \leq   \frac{2\alpha}{ 1 - e^{-2\alpha} }
  \Bigg[\frac{d^{(1)}_{\mbox{\tiny $S  \times   T$}} +1}{2}+ \frac{2\,\KL(\posterior\|\prior)  +  \ln  \frac{2}{\delta}}{m\times 2\alpha} \Bigg] .
\end{equation}

\smallskip
\noindent
We now bound
$$d^{(2)} \, \eqdef     \esp{(h,h')\sim \posterior^2 }    \Big[ \RDT(h,h')  -  \RDS(h, h') \Big]$$
using exactly the same argument as for $d^{(1)}$ except that we instead consider the following ``abstract''  loss of $\hat{h}$ on a pair of examples $(\xbs,\xbt)\sim \DST = D_S\times D_T$ :
$$\loss_{d^{(2)}} (\hat{h},\xbs  ,\xbt)  \ \eqdef \ \frac{1 +  \zoloss (h(\xbt),  h'(\xbt) -  \zoloss (h(\xbs ),  h'(\xbs ) )  )}{2}\,.$$
  We then obtain, with probability at least $1 - \tfrac{\delta}{2}$ over the choice of $S\times T\sim(\DS\times \DT)^m$,
\begin{equation} \label{eq:catoni_deuxio}
\frac{d^{(2)}+1}{2} \leq   \frac{2\alpha}{ 1 - e^{-2\alpha} }
  \Bigg[\frac{d^{(2)}_{\mbox{\tiny $S  \times   T$}} +1}{2}+ \frac{2\,\KL(\posterior\|\prior)  +  \ln  \frac{2}{\delta}}{m\times 2\alpha} \Bigg] .
\end{equation}
To finish the proof, note that by definition, we have that $d^{(1)}=-d^{(2)}$. Hence, we have 
$$|d^{(1)}| = |d^{(2)}| = \des(\DS,\DT),
\quad\mbox{ and }\quad 
|d^{(1)}_{\mbox{\tiny $S  \times   T$}}| = |d^{(2)}_{\mbox{\tiny $S  \times   T$}}| = \des(S,T).$$
Then, the maximum of the bound on $d^{(1)}$ (Equation~\eqref{eq:catoni_primo}) and the bound on $d^{(2)}$ (Equation~\eqref{eq:catoni_deuxio}) gives a bound on $\des(D_S, D_T)$.  
By the union bound,  with probability $1 - \delta$ over the choice of $S\times T\sim(\DS\times \DT)^m$, we have
\begin{equation*}
\frac{|d^{(1)}|+1}{2}\ \leq \  \frac{\alpha}{ 1 - e^{-2\alpha} }
  \Bigg[|d^{(1)}_{\mbox{\tiny $S  \times   T$}}| +1+ \frac{2\,\KL(\posterior\|\prior)  +  \ln  \frac{2}{\delta}}{m\times \alpha} \Bigg] ,
\end{equation*}
or, which is equivalent to
 \begin{align*}
\des(\DS,\DT)\ \leq\  
  \frac{2\alpha}{1 -e^{-2\alpha}} \left[ \des(S,T)  +  \frac{2\,\KL(\posterior\|\prior)  +  \ln  \frac{2}{\delta}}{m\times\alpha} + 1\right] - 1\,,
\end{align*}
and we are done.
\end{proof}

\section{Proof of Theorem \ref{thm:bound_dis_rho_allester}}
\label{app:mcallester}

\begin{proof}
Let us consider the non-negative random variable \quad   
$\esp{\mathclap{(h,h')\sim\prior^2}} e^{2m(\RDS(h,h')-\RS(h,h'))^2}.$

We apply Markov's inequality (Lemma~\ref{theo:markov}).
 For every $\delta\in(0,1]$, with a probability at least $1-\frac{\delta}{2}$ over the choice of $S\sim(\DS)^m$, we have
\begin{align}
\esp{(h,h')\sim\prior^2} \expo{2m(\RDS(h,h')-\RS(h,h'))^2} 
\  \leq&\ \nonumber
\frac{2}{\delta} \esp{S\sim (\DS)^m} \esp{(h,h')\sim\prior^2} \expo{2m(\RDS(h,h')-\RS(h,h'))^2}\\
=& \ \nonumber
\frac{2}{\delta} \esp{(h,h')\sim\prior^2} \esp{S\sim (\DS)^m} \expo{2m(\RDS(h,h')-\RS(h,h'))^2}\\
\label{2}
\leq& \  
\frac{2}{\delta} \esp{(h,h')\sim\prior^2} \esp{S\sim (\DS)^m} \expo{\kl(\RS(h,h')\,\|\,\RDS(h,h'))}\\
\label{3}\leq & \  
\frac{2}{\delta} \esp{(h,h')\sim\prior^2} 2\sqrt{m}\,.
\end{align}
Line~\eqref{2} comes from Pinsker's inequality, and Line~\eqref{3} comes from the Maurer's lemma (Lemma~\ref{lem:maurer}).
By taking the logarithm on each outermost side of the previous inequality,
we obtain
\begin{equation}
\ln \esp{(h,h')\sim\prior^2} \expo{2m(\RDS(h,h')-\RS(h,h'))^2} 
\  \leq\ 
\ln \frac{4\sqrt{m}}{\delta}\,.
\end{equation}

Let us now find a lower bound of the left side of the last equation by using 
the change of measure inequality (Lemma~\ref{lem:change-measure}) and the Jensen inequality (Lemma~\ref{theo:jensen}).
\begin{align*}
\ln \esp{(h,h')\sim\prior^2} &\expo{2m(\RDS(h,h')-\RS(h,h'))^2} \\
&\geq\  
\esp{(h,h')\sim\posterior^2} {2m(\RDS(h,h')-\RS(h,h'))^2}   -\KL({\posterior}^2\|{\prior}^2) \\
&\geq\
2m\left(\esp{(h,h')\sim\posterior^2} \RDS(h,h')-\esp{(h,h')\sim\posterior^2} \RS(h,h')\right)^2
-\KL({\posterior}^2\|{\prior}^2)\\
&=\ 
2m\Big( \RDS(\GQ,\GQ)-\RS(\GQ,\GQ)\Big)^2 -2\,\KL(\posterior\|\prior)\,.
 \end{align*} 
The last equality is obtained from Equation~\eqref{eq:abstract_gibbs_errors} and Lemma~\ref{lem:2KL}.
We finally obtain
\begin{equation*}
2m\Big( \RDS(\GQ,\GQ)-\RS(\GQ,\GQ)\Big)^2
\ \leq \
 2\,\KL(\posterior\,\|\,\prior) +  \ln  \frac{4\sqrt{m}}{\delta}\,,
\end{equation*}
and we conclude, with a probability at least $1-\frac{\delta}{2}$ over the choice of $S\sim(\DS)^m$, 
\begin{equation} \label{eq:super_ineq1} 
\bigg| \, \RDS(\GQ,\GQ)-\RS(\GQ,\GQ)\, \bigg|
\ \leq \ 
\sqrt{\frac{1}{2m} \left[2\,\KL(\posterior\,\|\,\prior) +  \ln  \frac{4\sqrt{m}}{\delta}\right]}  .
\end{equation}

\noindent
Following the exact same proof process with the random variable \quad
$\esp{\mathclap{(h,h')\sim\prior^2}}  e^{2m'(\RDT(h,h')-\RT(h,h'))^2}\!,$
we obtain, with a probability at least $1-\frac{\delta}{2}$ over the choice of $T\sim(\DT)^{m'}$, 
\begin{equation} \label{eq:super_ineq2} 
\bigg| \, \RDT(\GQ,\GQ)-\RT(\GQ,\GQ)\, \bigg|
\ \leq \ 
\sqrt{\frac{1}{2m'} \left[2\,\KL(\posterior\,\|\,\prior) +  \ln  \frac{4\sqrt{m'}}{\delta}\right]}  .
\end{equation}

Joining Inequalities~\eqref{eq:super_ineq1} and~\eqref{eq:super_ineq2} with the union bound (that assure that both results hold simultaneously with probability $1-\delta$), gives the result because
\begin{align*}
\bigg| \, \RDS(\GQ,\GQ)-\RDT(\GQ,\GQ)\, \bigg|& = \des(D_S, D_T)\,,\\
\bigg| \, \RS(\GQ,\GQ)-\RT(\GQ,\GQ)\, \bigg| &= \des(S, T)\,,
\end{align*}
and because  if  $|a_1 - b_1| \leq c_1$ and  $|a_2 - b_2| \leq c_2$, then
$ |(a_1-a_2)-(b_1-b_2)| \leq c_1 + c_2$. 
\end{proof}

\section{Proof of Theorem \ref{thm:bound_dis_rho_multi}}
\label{app:seeger_multi}

\begin{proof}
The proof follow all the steps of the proof of Theorem~\ref{thm:bound_dis_rho} (see Appendix~\ref{app:catoni}). 
The only difference is that, in order to obtain a guarantee over
$\des(\multiD,\DT)$, we bound 
\begin{align*}
\widehat{d^{(1)}} \, \eqdef  \esp{(h,h')\sim \posterior^2 }   \LB   \esp{\DSj\sim v} \RDSj(h,h') - \RDT (h,h') \RB
\end{align*}
by its empirical counterpart
\begin{align*}
\widehat{d^{(1)}_{\mbox{\tiny $\multiS\!\!\times\!\!T$}}} \, \eqdef  \esp{(h,h')\sim \posterior^2 }  \LB  \esp{\DSj\sim v} \RSj(h,h')  -  \RT(h,h') \RB\,.
\end{align*}
To do so, we define the ``abstract'' loss of $\hat{h}\eqdef (h,h')\in\Hcal^2 $ on a tuple of $n+1$ examples $(\xb^{s_1},\ldots, \xb^{s_n},\xbt)\sim D_{S_1}\times\ldots\times D_{S_n}\times  D_T$ by
$$\loss_{\widehat{d^{(1)}}} (\hat{h}, \xb^{s_1},\ldots, \xb^{s_n} ,\xbt)
\, \eqdef \,
 \frac{1}{2} \Bigg[ {1 + \esp{\DSj\sim v} \zoloss (h(\xb^{s_j} ),  h'(\xb^{s_j} ) ) -  \zoloss (h(\xbt),  h'(\xbt) )}\Bigg]\, .$$
Again, we obtain the result by following the proof of Theorem~\ref{thm:bound_dis_rho}.
\end{proof}

\section{Proof of Theorem~\ref{theo:pacbayesdabound_catoni_bis_multi_v}}

\label{proof:multi_v}

We first need the following result. 
\begin{theorem} \label{thm:bound_dis_rho_multi_v}
 For any distributions $\{\DSj\}_{j=1}^n$ and $\DT$ over $X$, any set of hypothesis $\Hcal$, for any prior distribution $u$ over $\{\DSj\}_{j=1}^n$,  any distribution $\prior$ over $\Hcal$, any $\delta \in (0,1]$, and any real number $\alpha > 0$,  with a probability at least $1 - \delta$ over the choice of  $\multiS\sim\multiD$, and $T\sim (D_T)^{m}$, for every distribution $v$ over $\{\DSj\}_{j=1}^n$, we have
 \begin{align*}
\desPI(\multiD,\DT)\, \leq\,   \frac{2\alpha }{1 -e^{-2\alpha}} \left[ \desPI(\multiS,T)  +  \frac{\KL(v\|u)  +  \ln  \frac{2}{\delta}}{n\times\alpha} + 1\right] - 1\,.
\end{align*}
\end{theorem}
\begin{proof}
The proof follows a process similar to the proof of Theorem~\ref{thm:bound_dis_rho} in Appendix~\ref{app:catoni}: we separately bound  
$$\RDT(\GQ,\GQ) -  \ \ \esp{\mathclap{\DSj\sim v}} \RDSj(\GQ,\GQ)
\quad\mbox{ and }\quad \ \ 
\esp{\mathclap{\DSj\sim v}} \RDSj(\GQ,\GQ) - \RDT(\GQ,\GQ)\,,$$
 by rescaling their value into $[0,1]$.
\end{proof}

\noindent
Then, we easily obtain the result of Theorem~\ref{theo:pacbayesdabound_catoni_bis_multi_v}.\\

\begin{proof}{\textbf{of Theorem~\ref{theo:pacbayesdabound_catoni_bis_multi_v}} }
In Theorem~\ref{thm:pacbayesdabound_multi}, replace $\RSv(G_\posterior)$ and $\des(\multiD,\DT)$ by their upper bound, obtained from Theorem~\ref{thm:pacbayescatoni} applied on $\RPv(G_\prior)=\mathrm{\bf E}_{\PSj\sim v}\,\RPSj(G_\prior)$ (instead of $\RPS(\GQ)$) and Theorem~\ref{thm:bound_dis_rho_multi_v}, 
with $\delta$ chosen respectively as $\frac{\delta}{3}$ and $\frac{2\delta}{3}$.
\end{proof}

\section{Proof of Theorem~\ref{theo:pacbayesdabound_catoni_bis_multi_vq}}
\label{proof:multi_vq_pascal}

\begin{proof}
Consider the data distribution $\Pcal \eqdef P_{S_1}\times P_{S_2} \times \ldots \times P_{S_n}$.
The loss of a classifier $h\in\Hcal$ on a tuple of examples $(\,(\xb_1,y_1),\ldots,(\xb_n,y_n)\,) \sim \Pcal$ is defined as the mean of the zero-loss $\zoloss\big( h(\xb_j), y_j \big)$ on each example of the tuple (\emph{i.e.}, $j\in\{1,\ldots,n\}$).

Thanks to this convention, and by  a slight abuse of notation, we can write the \emph{expected risk} on $\Pcal$ of a classifier $h\in\Hcal$  as
\begin{eqnarray*}
 R_\Pcal (h)  
 & \eqdef& 
 \esp{((\xb_1,y_1),\ldots,(\xb_n,y_n)) \sim \Pcal } \frac{1}{n} \sum_{j=1}^n   \zoloss\big( h(\xb_j), y_j \big)\\[-1mm]
&=& \frac{1}{n} \sum_{j=1}^n \RPSj(h)\,,
\end{eqnarray*}
and the \emph{expected disagreement} of a pair of classifiers $(h,h')\in\Hcal^2$ on the corresponding marginal distribution
 $\Dcal \eqdef D_{S_1}\times D_{S_2} \times \ldots \times D_{S_n}$
as
\begin{eqnarray*}
R_\Dcal(h,h')  
 & \eqdef& 
 \esp{(\xb_1,\ldots,\xb_n) \sim \Dcal } \frac{1}{n} \sum_{j=1}^n    \zoloss\big( h(\xb_j), h'(\xb_j) \big)\\
&=& \frac{1}{n} \sum_{j=1}^n \RDSj(h, h')\,.
\end{eqnarray*}

\noindent
Let now define new posterior ${\posterior_v}$ and prior ${\prior_u}$ on $\Hcal$:
\begin{equation*}
{\posterior_v} (h) \ =\ \posterior(h) \sum_{j=1}^n v(\PSj)
\quad \mbox{ and } \quad 
{\prior_u} (h) \ =\ \prior(h) \sum_{j=1}^n u(\PSj)\,.
\end{equation*}

\noindent
From above definitions, one can easily show 
\begin{equation*}
\RPv(\GQ)  = R_\Pcal(G_{{\posterior_v}})\,,
\quad \mbox{ and } \quad
\des(\multiD,\DT) = {\des}_v(\Dcal, \DT)\,.
\end{equation*}
Moreover, we have\\[-6mm]
\begin{eqnarray*}
\KL(\posterior_v\|\prior_u) 
& =& 
\esp{h\sim\posterior_v } \ln \frac{\posterior_v(h)}{\prior_u(h)} \\
& =& 
\esp{h\sim\posterior }  \sum_{j=1}^n v(\PSj) \left[\ln \frac{\posterior(h)}{\prior(h)} + \ln\frac{v(\PSj)}{u(\PSj)}\right] \\
& =& 
\esp{h\sim\posterior } \ln \frac{\posterior(h)}{\prior(h)} +  \sum_{j=1}^n v(\PSj)\ln\frac{v(\PSj)}{u(\PSj)} \\
&=&  \KL(\posterior\|\prior) + \KL(v\|u)\,.
\end{eqnarray*}

 From Theorem~\ref{theo:pacbayesdabound_catoni_bis},  with a probability at least $1-\delta$ over the choice of $\Scal  \times  T  \sim (\Pcal \times  D_T)^m $, for every posterior distribution $\posterior_v$ on $\Hcal$, we have
 \begin{align*} 
\RPT(G_\posterior)  
\ \leq\  c'\, R_\Scal(G_{\posterior_v})  +  \alpha'\, \tfrac{1}{2} {\des}_v(\Scal,T) + 
  \left( \frac{c'}{c} + \frac{\alpha'}{\alpha} \right)  \frac{\KL(\posterior_v\|\prior_u)+\ln\frac{3}{\delta}}{m}
   + \lambda_{\posterior_v} + \tfrac{1}{2} (\alpha'\!-\! 1)
   \,,
 \end{align*}
and we obtain the final result by the substitution of $R_\Scal(G_{\posterior_v})$, ${\des}_v(\Scal,T)$, and $\KL(\posterior_v\|\prior_u)$ with their equivalent expression.
\end{proof}


\vskip 0.2in
\bibliography{biblio}

\end{document}